\documentclass[lettersize,journal]{IEEEtran}
\usepackage{amsmath,amssymb,amsfonts}
\usepackage{array}
\usepackage[caption=false,font=normalsize,labelfont=sf,textfont=sf]{subfig}
\usepackage{textcomp}
\usepackage{stfloats}
\usepackage{verbatim}
\usepackage{url}
\usepackage{graphicx}
\usepackage[table]{xcolor}
\usepackage{cite}
\usepackage{multirow}
\usepackage{booktabs}
\usepackage{amsthm}
\usepackage{algorithm}
\usepackage{algorithmicx}
\usepackage{algpseudocode}

\newtheorem{theorem}{Theorem}

\newcommand{\Mref}{M_{\mathrm{ref}}}
\newcommand{\Mpro}{M_{\mathrm{pro}}}
\definecolor{sharedbg}{RGB}{226,241,252}
\definecolor{diffbg}{RGB}{254,235,220}
\newcommand{\groupbest}[1]{\textbf{#1}}
\newcommand{\shared}[1]{\begingroup\setlength{\fboxsep}{1pt}\colorbox{sharedbg}{\strut #1}\endgroup}
\newcommand{\diffpart}[1]{\begingroup\setlength{\fboxsep}{1pt}\colorbox{diffbg}{\strut #1}\endgroup}
\graphicspath{{figures/}}

\begin{document}

\title{CBD: API-Only LLM Black-Box Unlearning through Controlled Behavioral Divergence}

\author{Zhiqiang~Xie,~Yijing~Lin,~Zhipeng~Gao,~and~Dong~In~Kim,~\IEEEmembership{Fellow,~IEEE}%
\thanks{Corresponding author: Yijing Lin.}%
\thanks{Zhiqiang Xie, Yijing Lin, and Zhipeng Gao are with the State Key Laboratory of Networking and Switching Technology, Beijing University of Posts and Telecommunications. Yijing Lin is also with the Beijing Advanced Innovation Center for Future Blockchain and Privacy Computing. Email: \{xiezhiqiang, yjlin, gaozhipeng\}@bupt.edu.cn. Dong In Kim is with the Department of Electronic and Electrical Engineering, Sungkyunkwan University, Suwon, South Korea. Email: dongin@skku.edu.}%
\thanks{This work is supported by the National Natural Science Foundation of China (62502041, 92467203, 62372050), the Beijing Advanced Innovation Center for Future Blockchain and Privacy Computing, the Beijing Natural Science Foundation (L251038, L244010), the CCF-Huawei Populus Grove Fund (TC202418), the Fellowship of China National Postdoctoral Program for Innovative Talents (BX20240045), and the China Postdoctoral Science Foundation General Program (2025M773481).}}

\maketitle
\raggedbottom

\begin{abstract}
Edge devices increasingly invoke large language models (LLMs) through API services for context-aware edge intelligence, while edge-generated data may be collected to improve foundation LLMs and may introduce sensitive, copyrighted, harmful, or outdated information into model behavior.
Machine unlearning has therefore become a practical way to remove the influence of undesired data without retraining LLMs from scratch.
However, existing LLM unlearning methods still face two key gaps in practical deployment. The first is how to achieve unlearning under API-only black-box access, where target-model parameters and internal logits are unavailable. The second is how to preserve retained utility when unlearning target data and retained data share highly similar prompt structures or semantic patterns. 
To address these challenges, we propose Controlled Behavioral Divergence (CBD), an API-only black-box unlearning framework.
Specifically, CBD uses two auxiliary models to create controlled behavioral divergence between retained inputs and unlearning target inputs, converts this divergence into an unlearning-relevance score, and routes unlearning-related prompts away from the target LLM. 
Furthermore, to improve discrimination accuracy under high similarity between unlearning target data and retained data, CBD constructs a gradient-statistics-based discriminative basis by estimating empirical Fisher matrices and solving a regularized generalized eigenvalue problem, guiding the unlearning signal toward target-specific information rather than shared prompt structures. 
Compared with eleven representative white-box and gray-box unlearning baselines across multiple benchmark settings, CBD achieves a better unlearning-utility trade-off and its performance varies little across hyperparameter settings. On ToFU forget10, CBD approaches the retrained reference on the forget set while raising model utility to 74.90, about 15\% above the second-best baseline. On WMDP, it lowers hazardous-knowledge accuracy to 25.68, near the random-guess level, while preserving an MMLU accuracy of 52.67. Our code is available at \url{https://github.com/DGL-codes/CBD}.
\end{abstract}


\newpage
\section{Introduction}

\IEEEPARstart{L}{arge} language models are increasingly invoked by edge devices through API services to support context-aware edge intelligence under resource-constrained conditions~\cite{qu2025mobile}.
In this workflow, user queries, feedback, and interaction logs generated at the edge may be collected to improve the performance of the foundation LLM in edge-specific scenarios. 
However, such use of edge data may inadvertently introduce sensitive user data, copyrighted materials, and potentially harmful or outdated information into the LLM's parametric memory~\cite{carlini2021extracting}. 
Although retraining the foundation LLM from scratch after removing the undesired data provides the most thorough solution, it is computationally prohibitive due to the massive parameter scale and training corpus of modern LLMs.
In this context, machine unlearning has been proposed to selectively remove such influence while preserving model utility on downstream tasks~\cite{yao2024machine}.

\begin{figure}[!t]
    \centering
    \includegraphics[width=0.92\columnwidth]{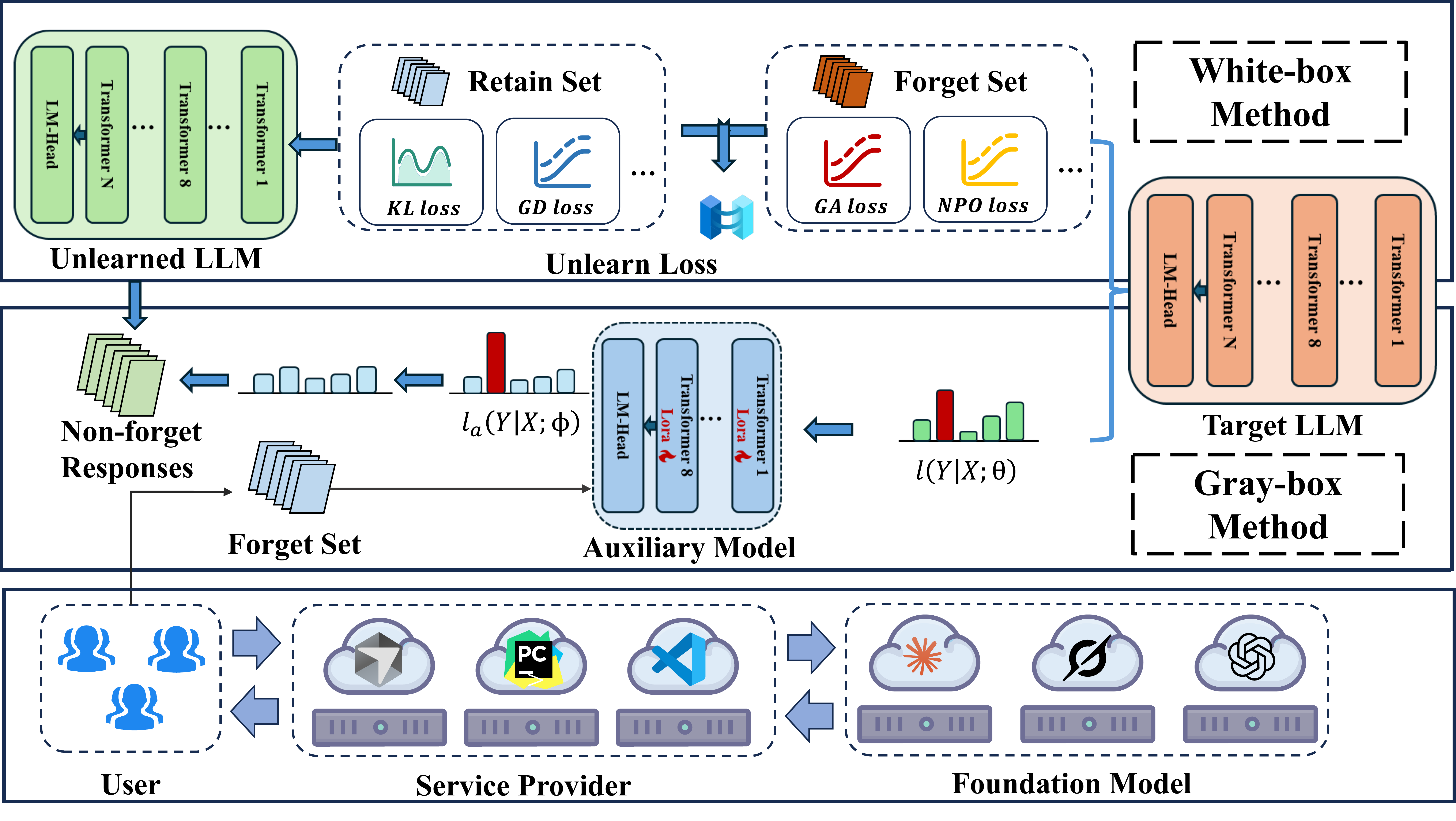}
    \caption{Architecture comparison of machine unlearning methods for LLMs.}
    \label{fig:intro}
\end{figure}

Machine unlearning for LLMs can be viewed as a multi-objective optimization problem~\cite{liu2025rethinking}, as it requires simultaneously removing the influence of target data while preserving model performance on retained data. Existing approaches fall into two groups, as illustrated in Figure~\ref{fig:intro}. First, continual-training-based unlearning methods employ bidirectional gradient optimization strategies, performing gradient ascent or its variants on the forget set to reverse parameter updates induced by forget data, while applying Kullback-Leibler (KL) divergence or gradient descent on the retain set to preserve model performance on retained data~\cite{yao2024machine,yao2024large,zhang2024negative}. Second, auxiliary-model-based methods introduce additional models to calibrate the log probabilities of target models, thereby reducing the output probabilities of forget-related tokens at lower training cost. Representative methods include Unlearning from Logit Difference (ULD)~\cite{ji2024reversing} and offset unlearning~\cite{huang2024offset}.
However, current methods have yet to address black-box LLM unlearning, and they still degrade retained utility when forget and retain data are highly similar. In this paper, we aim to overcome the following two challenges.

The first challenge is how to achieve LLM unlearning in a black-box setting where the target model is exposed only through API calls. In the edge-service workflow described above, edge devices and downstream service components usually receive only final responses from the LLM, without access to its parameters, gradients, or internal token probabilities. This restriction makes white-box methods that update target-model parameters~\cite{yao2024machine,yao2024large,zhang2024negative} and gray-box methods that correct target-model logits~\cite{ji2024reversing,huang2024offset} inapplicable. Beyond the interface limitation, updating or fine-tuning a deployed foundation LLM often involves validation, approval, and redeployment procedures, which makes it difficult to satisfy time-sensitive data removal requests. Therefore, API-only edge LLM services require an unlearning mechanism that can reduce the exposure of undesired data without retraining or editing the target model, or inspecting its internal states.

The second challenge is how to preserve retained utility while unlearning target data. In many LLM unlearning scenarios, the data requested for unlearning and the retained data may share similar topics, prompt templates, or semantic structures, differing only in specific entities or attributes. As illustrated in Table~\ref{tab:high_similarity_examples}, such high structural similarity makes unlearning targets and retained samples difficult to separate, so updates intended to suppress the target data may also alter the model behavior associated with retained data~\cite{yao2024machine,yao2024large,zhang2024negative}. As the unlearning set grows, the unlearning algorithm may increasingly attribute the removal target to shared prompt templates rather than to target-specific content, degrading performance on retained samples with similar prompt structures. Therefore, effective LLM unlearning requires a mechanism that separates target-specific content from shared prompt structures, thereby preserving retained utility during the unlearning process.

To address the above challenges, we propose \textbf{C}ontrolled \textbf{B}ehavioral \textbf{D}ivergence (CBD), a black-box unlearning framework that operates in a fully API-only setting.
CBD trains a probe model against a frozen reference model and estimates the relevance of an input prompt to the unlearning target from the behavioral divergence between the two models.
Inspired by Gradient Projection Memory~\cite{saha2021gradient}, we constrain the probe update so that the two auxiliary models remain close on retained inputs while becoming distinguishable on unlearning target inputs, and the resulting unlearning-relevance score routes unlearning-related prompts away from the target LLM.
Because a retain-only constraint becomes unreliable when unlearning target data and retained data are highly similar, CBD further estimates unlearning-side and retain-side empirical Fisher matrices from sample gradients, formulates direction selection as a regularized generalized eigenvalue problem, and restricts the probe update to directions that induce strong behavioral change on target data at limited cost on retained data.
The main contributions of this paper are summarized as follows.
\begin{itemize}
\item We propose CBD, an API-only black-box unlearning framework that reduces the exposure of the data requested for removal without accessing target-model parameters, gradients, or internal logits.
\item We design a dual-auxiliary-model mechanism that uses controlled behavioral divergence to identify unlearning-related prompts and route them away from the target LLM.
\item We develop a gradient-statistics-based discriminative basis extraction method that improves discrimination accuracy under high similarity between unlearning target data and retained data, thereby preserving retained utility.
\item We evaluate CBD on benchmark unlearning datasets against white-box and gray-box baselines, showing an improved unlearning-utility trade-off under the API-only black-box setting.
\end{itemize}

The remainder of this paper is organized as follows. Section~\ref{sec:related} reviews related work, and Section~\ref{sec:preliminaries} revisits the white-box and gray-box paradigms and formalizes the API-only setting. Section~\ref{sec:method} presents the query routing procedure of CBD, and Section~\ref{sec:subspace} develops the discriminative Fisher basis for the high-similarity regime. Section~\ref{sec:exp} reports the experiments, and Section~\ref{sec:conclusion} concludes the paper.

\section{Related Work}
\label{sec:related}

Machine unlearning aims to remove the influence of designated data from a trained model so that the resulting model behaves as if the data had never been used during training~\cite{bourtoule2021machine}. For LLMs, this goal is difficult to achieve by exact retraining because the original corpus is large, the training cost is high, and the learned representations are highly entangled~\cite{yao2024machine,yao2024large,liu2025rethinking}. Existing LLM unlearning methods can be organized by the interface they assume, where the first line directly updates the target LLM and the second line avoids full target retraining by using auxiliary models, logits, prompts, embeddings, or in-context data at inference time. CBD is closest to the second line in motivation, but targets a stricter API-only setting in which the target LLM provides only final responses and is not logit-observable.

\subsection{White-Box Machine Unlearning}
White-box LLM unlearning directly changes the target model after an unlearning request arrives. A common starting point is to reverse the training signal on the forget data. Jang et al.~\cite{jang2023knowledge} showed that gradient ascent on target token sequences can reduce memorized private knowledge, which made direct loss maximization a standard baseline. Chen and Yang~\cite{chen2023forget} introduced lightweight unlearning layers and selective teacher-student objectives to reduce the cost of full-model updates. Yao et al.~\cite{yao2024machine} further benchmarked several first-order unlearning strategies for pre-trained LLMs. Related studies formulate deletion objectives for sensitive information extraction attacks~\cite{patil2024sensitive} and practical knowledge unlearning settings~\cite{tian2024forget}, showing that the unlearning-utility trade-off is sensitive to the objective, data split, and hyperparameters.

Subsequent work improves this direct-update paradigm by changing the optimization objective or by selecting more targeted update components. Large Language Model Unlearning~\cite{yao2024large} treats the problem as a broader LLM safety task and examines the conflict between removing specific knowledge and preserving general capability. Negative Preference Optimization (NPO)~\cite{zhang2024negative} replaces pure ascent with a preference-style objective, and later work revisits this objective to reduce reference-model bias and simplify the update rule~\cite{fan2025simplicity}. Other methods use second-order information~\cite{jia2024soul}, strategic weight attribution~\cite{jia2024wagle}, uniform-target self-distillation~\cite{vasilev2025unilogit}, or general enhancement frameworks for fine-tuning-based unlearning~\cite{ren2025framework}. Continual and domain-specific variants further study repeated requests~\cite{gao2025continual}, copyright removal~\cite{dou2025copyright}, stealthy knowledge concealment~\cite{gu2025concealment}, selective token-level forgetting~\cite{wan2025token}, self-generated forget data~\cite{xie2025reveal}, and belief-space rectification for reasoning~\cite{niwa2025belief}. These methods also motivate the white-box baselines used in this paper. However, their implementation premise is still white-box access, since the service provider must be able to update target-model parameters, compute gradients, and validate the edited model before redeployment. Such access is unavailable when the LLM is consumed only through an external API.

\subsection{Auxiliary and Inference-Time Machine Unlearning}
Another line of work reduces or avoids direct target-model editing by moving the unlearning effect to auxiliary branches or inference-time mechanisms. ULD~\cite{ji2024reversing} trains an assistant LLM with reversed forget-retain objectives and obtains the unlearned distribution through the logit difference between the target model and the assistant model. Offset unlearning~\cite{huang2024offset} learns a logit offset from a pair of smaller models and transfers the correction to black-box LLM services. Both are more deployment-oriented than direct retraining, but their unlearning effect is still expressed through correction distributions or offsets learned from auxiliary logits, rather than through a pure query-response routing mechanism.

Other methods operate on the prompt or context instead. In-context unlearning~\cite{pawelczyk2024incontext} shows that selected examples with altered labels can induce instance-level forgetting without parameter updates, and later work shows that such in-context control may conceal rather than truly remove knowledge~\cite{takashiro2025answer}. Soft Prompting for Unlearning (SPUL)~\cite{bhaila2025spul} learns prompt tokens that are prepended to the input to enforce forgetting and preserve retained utility. ECO prompts~\cite{liu2024eco} use a classifier to identify prompts in the scope of unlearning and then apply learned corruptions in the prompt-embedding space. Fast exact unlearning for in-context learning data~\cite{muresanu2025fast} studies removal from an external in-context data mechanism rather than from model weights. These methods show that inference-time control can replace retraining when retraining is impractical, but they still rely on structured prompt control, soft prompts, embedding manipulation, or an explicit in-context data store.

CBD differs from the above approaches in two aspects. First, it does not edit the target LLM and does not require target-model logits, gradients, embeddings, or soft-prompt access, since it only uses auxiliary models to score incoming queries and route unlearning-related ones away from the target API. Second, CBD explicitly addresses the retained-utility problem that arises when forget and retain data share similar prompt structures, designing its discriminative basis so that the auxiliary behavioral divergence reflects target-specific information rather than shared templates. This combination of API-only access and retain-aware discrimination is the gap addressed in this paper.

\section{Preliminaries and Framework}
\label{sec:preliminaries}

We first review the two mainstream unlearning paradigms and clarify why the API-only setting requires a different solution. We use $M$ to denote the target LLM, $\Mref$ and $\Mpro$ to denote the reference and probe models, $D_f$ and $D_r$ to denote the forget and retain sets, and $w_0$ and $w$ to denote the initial and current trainable auxiliary parameters. The main symbols are summarized in Table~\ref{tab:symbols}.

\begin{table}
    \centering
    \caption{Key notations.}
    \label{tab:symbols}
    \setlength{\tabcolsep}{2pt}
    \footnotesize
    \begin{tabular}{@{}>{\centering\arraybackslash}m{0.14\columnwidth} >{\centering\arraybackslash}m{0.20\columnwidth} >{\centering\arraybackslash}m{0.14\columnwidth} >{\centering\arraybackslash}m{0.20\columnwidth}@{}}
        \toprule
        Notation & Definition & Notation & Definition \\
        \midrule
        $M$ & target LLM & \shortstack{$\Mref$\\$\Mpro$} & reference and probe models \\
        $D_f$ & forget set & $D_r$ & retain set \\
        $\theta$ & target-model parameters & $w_0,w$ & auxiliary parameters \\
        $d_{\mathrm{in}},d_{\mathrm{out}}$ & layer input and output dimensions & $d_w$ & trainable parameter dimension \\
        $\mathcal{L}_f,\mathcal{L}_r$ & probe training objectives & $\eta$ & learning rate \\
        $z_M(x)$ & target-model logits & $z_{\mathrm{ast}}(x)$ & assistant logits \\
        $\widetilde{z}(x)$ & modified logits & $A_m,B_m$ & LoRA factors \\
        $\bar g_f$ & projected forget gradient & $g_f,g_r$ & training gradients \\
        $F_f,F_r$ & average Fisher matrices & $G_f,G_r$ & sample-gradient matrices \\
        $U_r,Q$ & retain and discriminative bases & $\mu$ & damping factor \\
        $d_t(x)$ & local divergence & $s(x)$ & unlearning-relevance score \\
        $\mathcal{T}(x)$ & aligned positions & $\tau$ & routing threshold \\
        $\Gamma$ & threshold candidates & $\mathcal{J}$ & calibration objective \\
        $\beta$ & retain-side weight & $\rho$ & retain-side budget \\
        \bottomrule
    \end{tabular}
\end{table}

\subsection{White-Box and Gray-Box Unlearning Methods}
White-box methods update the already trained target model with parameters $\theta$ directly. The simplest strategy is gradient ascent (GA), which increases the answer-side loss on the forget set to push the model away from the data requested for removal~\cite{jang2023knowledge,yao2024machine}. However, pure ascent often degrades performance on retained data severely. To preserve retained performance while forgetting, most white-box methods therefore combine a forget loss with a retain loss into a single objective,
\begin{equation}
    \mathcal{L}_{\mathrm{wb}}(\theta)
    =
    \mathcal{L}_f^{\mathrm{wb}}(\theta,D_f) + \lambda\,\mathcal{L}_r^{\mathrm{wb}}(\theta,D_r),
    \label{eq:white_box_obj}
\end{equation}
where the forget loss $\mathcal{L}_f^{\mathrm{wb}}$ drives the model away from the forget data, the retain loss $\mathcal{L}_r^{\mathrm{wb}}$ preserves behavior on retained data, and $\lambda$ trades off the two. Beyond GA, the forget loss $\mathcal{L}_f^{\mathrm{wb}}$ can be instantiated by NPO, which mitigates the utility loss of pure ascent~\cite{zhang2024negative}, by a preference objective such as Direct Preference Optimization (DPO)~\cite{rafailov2023direct}, or by other targeted update rules~\cite{chen2023forget,kassem2023dememorization,wang2024llm}. The retain loss $\mathcal{L}_r^{\mathrm{wb}}$ is typically a gradient-descent term on retained data, giving the +GD variants, or a KL term that keeps the output distribution close to that of the original model, giving the +KL variants~\cite{yao2024machine}. All of these methods, however, require the target LLM to remain editable throughout unlearning, so this entire paradigm becomes unavailable once the target model is exposed only through an API.

Gray-box methods avoid editing the target model but still rely on its token-level logits. ULD trains an assistant model with reversed unlearning objectives so that the assistant concentrates the behavior to be removed, and subtracts the assistant logits $z_{\mathrm{ast}}(x)$ from the target logits $z_M(x)$~\cite{ji2024reversing},
\begin{equation}
    \widetilde{z}(x)=z_M(x)-\alpha z_{\mathrm{ast}}(x),
    \label{eq:uld}
\end{equation}
where $\alpha$ is a scaling coefficient, so tokens related to the data requested for removal become less likely to be selected during decoding. Offset unlearning maintains a pair of smaller auxiliary models and injects their logit difference as a transferable correction~\cite{huang2024offset},
\begin{equation}
    \widetilde{z}(x)=z_M(x)+\alpha\bigl(z_{\mathrm{ref}}(x)-z_{\mathrm{upd}}(x)\bigr),
    \label{eq:offset}
\end{equation}
where $z_{\mathrm{ref}}(x)$ and $z_{\mathrm{upd}}(x)$ are the logits of the frozen and the unlearning-updated auxiliary models, so the correction shifts the target distribution away from the behavior learned by the updated auxiliary branch. Both corrections exist only at the logit interface, so once the target LLM returns final text responses alone, neither~\eqref{eq:uld} nor~\eqref{eq:offset} can be applied. The next subsection formalizes this API-only setting.

\subsection{API-Only Scenario and Proposed Framework}

\begin{figure*}[!t]
    \centering
    \includegraphics[width=0.86\textwidth]{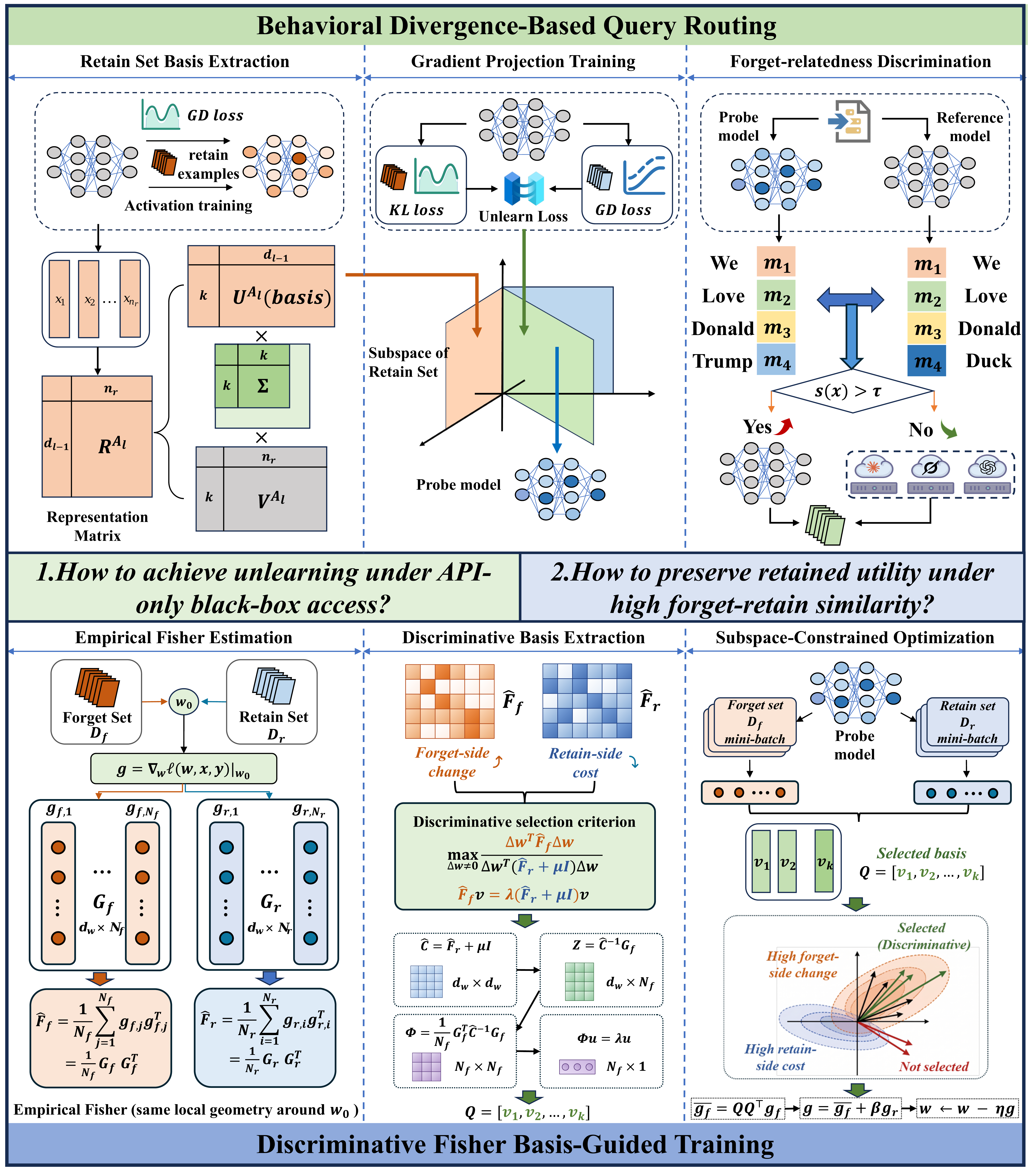}
    \caption{Framework of Controlled Behavioral Divergence (CBD) for black-box unlearning.}
    \label{fig:framework}
\end{figure*}

We consider an API-only deployment in which the target LLM $M$ is accessible only through query-response interaction. The unlearning request specifies a forget set $D_f=\{(x_j,y_j)\}_{j=1}^{N_f}$ and a retain set $D_r=\{(x_i,y_i)\}_{i=1}^{N_r}$. Under this restriction, the two conventional paradigms described above are both unavailable. White-box methods cannot be used because the target model cannot be edited directly, and gray-box methods cannot be used because the internal logit interface does not exist. Consequently, the problem is no longer how to overwrite target-model parameters, but how to achieve black-box unlearning using only the API response interface.

CBD uses three model roles (as illustrated in Figure~\ref{fig:framework}). The target model $M$ remains unchanged and continues to serve ordinary queries. A frozen reference model $\Mref$ provides the reference behavior. A trainable probe model $\Mpro$ is optimized to remain close to $\Mref$ on retained inputs but separate from $\Mref$ on forget inputs. Note that $\Mpro$ is not trained as a replacement for the target LLM but works with $\Mref$ to score whether an input is associated with the forget set. Both auxiliary models are initialized from a small base model that has never been trained on $D_f$, so the responses of $\Mref$ are less likely to reproduce the forget data directly, although this does not guarantee that $\Mref$ is free of related knowledge acquired during its own pretraining. CBD thus reduces the exposure of the data requested for removal at the service interface rather than erasing it from the target parameters, assuming the routing layer mediates all access to the target LLM.

The procedure proceeds in four steps. We first train $\Mpro$ under forget and retain objectives, then convert the divergence between $\Mref$ and $\Mpro$ into an unlearning-relevance score $s(x)$, calibrate a threshold $\tau$ on validation forget and retain samples, and finally route each query according to
\begin{equation}
    \mathrm{CBD}(x)=
    \begin{cases}
        \Mref(x), & s(x)>\tau,\\
        M(x), & s(x)\le \tau.
    \end{cases}
    \label{eq:routing}
\end{equation}
The remaining question is how to make this routing rule reliable when forget and retain data are highly similar, because weak separation directly undermines black-box unlearning. Section~\ref{sec:method} defines the CBD procedure from probe training to query routing. Section~\ref{sec:subspace} defines the discriminative basis used by the projected probe update.

\section{Behavioral-Divergence-Based Query Routing}
\label{sec:method}

Since an API-only LLM service exposes only final responses, unlearning methods that edit target-model parameters or correct target logits cannot be applied directly. To address this limitation, CBD adopts a dual-auxiliary-model mechanism for black-box unlearning, which uses the behavioral divergence between a frozen reference model $\Mref$ and a trainable probe model $\Mpro$ to identify unlearning-related queries during inference. Based on the resulting unlearning-relevance score, CBD routes unlearning-related queries to $\Mref$ and other queries to the target LLM $M$, thereby reducing the exposure of the data requested for removal without changing the target model. When forget and retain inputs are highly similar, the basic divergence signal may become insufficiently separable, which motivates the discriminative Fisher basis developed in Section~\ref{sec:subspace}.

\subsection{Dual Auxiliary Models and Probe Objective}

The reference model $\Mref$ is kept fixed as the behavioral reference, whereas the probe model $\Mpro$ is updated so that forget-related inputs become more distinguishable from retained inputs when compared with $\Mref$. To keep the auxiliary computation lightweight, both models are implemented with LoRA adapters~\cite{hu2022lora}, while only the probe-side adapters are updated during training. For a linear layer with frozen weight matrix $W_0\in\mathbb{R}^{d_{\mathrm{out}}\times d_{\mathrm{in}}}$, the LoRA parameterization is
\begin{equation}
    W = W_0 + B A,
    \label{eq:lora}
\end{equation}
where $A\in\mathbb{R}^{r\times d_{\mathrm{in}}}$ and $B\in\mathbb{R}^{d_{\mathrm{out}}\times r}$ are trainable low-rank factors with $r\ll\min(d_{\mathrm{in}},d_{\mathrm{out}})$. The base weight $W_0$ is shared across both auxiliary models, while the trainable parameters $w=\{A_m,B_m\}_m$ distinguish them. Let $w_0$ denote the initial trainable parameters inherited from the base model. The reference model $\Mref$ keeps $w_0$ unchanged, whereas the probe model $\Mpro$ updates $w$ through a controlled training process.

All subsequent geometric quantities, including gradients, Fisher matrices, and subspace bases, are defined in this trainable parameter subspace $w\in\mathbb{R}^{d_w}$ rather than in the full parameter space of the target LLM. In our implementation, basis extraction and projection act on the LoRA factors of the up-projection layers, which in our experiments sufficed to separate forget inputs from retain inputs.

For a supervised sample $(x,y)$, we define the answer-side negative log-likelihood as
\begin{equation}
    \ell(w, x, y)= - \sum_{t \in \mathrm{ans}} \log p_w(y_t \mid x, y_{<t}),
    \label{eq:nll}
\end{equation}
where $p_w(\cdot \mid x)$ is the output distribution of the probe model and the summation is over answer-token positions. On the forget side, we optimize
\begin{equation}
    \mathcal{L}_f(w)
    =
    \mathbb{E}_{(x,y)\sim D_f}[\ell(w, x, y)],
    \label{eq:forget_loss}
\end{equation}
which increases the behavioral separation between $\Mpro$ and $\Mref$ on the forget set. On the retain side, we align the predictive trajectory of $\Mpro$ to that of $\Mref$ through a sequence-level KL divergence. Let $z=[x,y]$ denote the full retain sequence of length $T$ and let $h_t=z_{<t}$ be its prefix at position $t$. The retain-side objective is
\begin{equation}
    \mathcal{L}_r(w)
    =
    \mathbb{E}_{(x,y)\sim D_r}
    \left[
        \frac{1}{T}\sum_{t=1}^{T}
        \mathrm{KL}\!\left(
            p_{w_0}(\cdot \mid h_t)\,\|\,p_w(\cdot \mid h_t)
        \right)
    \right],
    \label{eq:retain_loss}
\end{equation}
which acts on the full token-level predictive path rather than only on answer tokens. This term prevents the probe model from drifting away from the reference behavior on retained queries. The total probe-training objective is
\begin{equation}
    \mathcal{L}(w)=\mathcal{L}_f(w)+\beta \mathcal{L}_r(w),
    \label{eq:total_loss}
\end{equation}
where $\beta$ balances separation on the forget set against preservation on retained data.

\subsection{Activation-Informed Retain Basis as an Initial Approximation}

The objectives above specify how the probe model should move, but they do not yet determine which update directions should be protected so that behavioral divergence appears mainly on the forget set. CBD does not aim to move $\Mpro$ uniformly across all inputs. Instead, we want $\Mpro$ to separate from $\Mref$ on forget inputs while staying close to $\Mref$ on retained behavior. Inspired by Gradient Projection Memory (GPM) for continual learning~\cite{saha2021gradient}, we therefore begin with a retain-side activation basis as an initial approximation to the directions that should not be disturbed.

Consider the forward computation at layer $m$ under LoRA parameterization,
\begin{equation}
    y_m = (W_m + B_m A_m)x_{m-1},
    \label{eq:lora_layer}
\end{equation}
where $W_m$ is the frozen base weight, $A_m$ and $B_m$ are the trainable LoRA factors, and $x_{m-1}$ is the input activation to layer $m$. To make the dependence on the activation explicit, we use a squared-error surrogate at layer $m$,
\begin{equation}
    L_m
    =
    \frac{1}{2}\|y_m-y_m^\star\|_2^2
    =
    \frac{1}{2}\|(W_m+B_mA_m)x_{m-1}-y_m^\star\|_2^2,
    \label{eq:lora_surrogate}
\end{equation}
where $y_m^\star$ denotes the target output. The same dependence on $x_{m-1}$ also holds for the standard cross-entropy loss used in language modeling. Let $\delta_m=(W_m+B_mA_m)x_{m-1}-y_m^\star$ denote the output error. Then the LoRA gradients are
\begin{equation}
    \frac{\partial L_m}{\partial A_m}
    =
    B_m^{\top} \delta_m x_{m-1}^{\top},
    \qquad
    \frac{\partial L_m}{\partial B_m}
    =
    \delta_m (A_m x_{m-1})^{\top}.
    \label{eq:lora_grads}
\end{equation}
These expressions show that the update directions are determined by the input activation $x_{m-1}$ and its linear transforms. Dominant activation directions that repeatedly appear across retain samples therefore form an initial approximation to the directions along which retained behavior is most sensitive.

Using retain samples, we collect the layer-$m$ activation representations
\begin{equation}
    R_m
    =
    [x_{m-1,1},x_{m-1,2},\dots,x_{m-1,n_r}]
    \in\mathbb{R}^{d_{m-1}\times n_r},
    \label{eq:retain_repr}
\end{equation}
where $x_{m-1,i}$ is the input activation of the $i$-th retain sample at layer $m$, $n_r$ is the number of collected retain activations, and $d_{m-1}$ is the activation dimension. We compute the singular value decomposition
\begin{equation}
    R_m = U_m \Sigma_m V_m^{\top},
    \label{eq:repr_svd}
\end{equation}
and choose the smallest rank $k_m$ such that
\begin{equation}
    \|(R_m)_{k_m}\|_F^2 \ge \epsilon_m \|R_m\|_F^2,
    \label{eq:svd_energy}
\end{equation}
where $(R_m)_{k_m}$ denotes the rank-$k_m$ approximation and $\epsilon_m$ is the retained-energy threshold. The leading left singular vectors form the retain-side activation basis
\begin{equation}
    U_{r,m}=[u_{m,1},u_{m,2},\dots,u_{m,k_m}],
    \label{eq:retain_basis_matrix}
\end{equation}
\begin{equation}
    \mathcal{S}_{r,m}=\mathrm{span}\{u_{m,1},u_{m,2},\dots,u_{m,k_m}\}.
    \label{eq:retain_basis}
\end{equation}
Each $u_{m,i}$ is a direction in the feature space, so it describes a co-activation pattern that repeatedly appears across retain samples, whereas the right singular vectors in $V_m$ describe how individual samples combine along those patterns. Since the purpose of projection is to preserve stable retain-side activation directions rather than sample-specific identities, the left singular vectors provide the appropriate basis.

Since the rows of $\partial L_m/\partial A_m$ in~\eqref{eq:lora_grads} lie in the layer-$m$ activation space, each layer-wise basis projects the gradient by right multiplication with $I-U_{r,m}U_{r,m}^{\top}$. Stacking the layer-wise retain projectors $U_{r,m}U_{r,m}^{\top}$ over the vectorized trainable parameters yields a block-diagonal orthogonal projector $U_rU_r^{\top}$ on $\mathbb{R}^{d_w}$, where $U_r$ is the global retain basis. The retain-only orthogonal projection then removes forget-side motion along retain-sensitive directions,
\begin{equation}
    \bar{g}_f^{(r)} = g_f - U_r U_r^{\top} g_f,
    \label{eq:retain_orth_proj}
\end{equation}
where $g_f=\nabla_w\mathcal{L}_f(w)$ is the forget-side gradient. This update keeps the forget-induced change away from directions that encode retained behavior. Writing $g_r=\nabla_w \mathcal{L}_r(w)$ for the retain-side gradient, the corresponding base training direction is
\begin{equation}
    g^{(r)} = \bar{g}_f^{(r)} + \beta g_r,
    \label{eq:combined_grad_retain}
\end{equation}
and the probe parameters are updated by
\begin{equation}
    w \leftarrow w - \eta g^{(r)}.
    \label{eq:param_update_retain}
\end{equation}
When forget-retain coupling is weak, as on the smallest split of the Task of Fictitious Unlearning (ToFU) benchmark, the projected update already yields useful separation between forget-like and retained inputs. On that split, this base variant still attains a routing area under the receiver operating characteristic (ROC) curve (AUC) of 0.9559 (Section~\ref{sec:exp}). The forget-side term avoids retain-sensitive directions, whereas the retain-side term keeps the probe model aligned with the reference behavior.

\subsection{Behavioral Divergence Score and Routing Rule}

After probe training, we convert the behavioral divergence between $\Mref$ and $\Mpro$ into an unlearning-relevance score $s(x)$, where a larger value indicates stronger relevance to the forget set. The divergence is quantified using symmetric KL divergence at the token level. For an input $x$, let $p_{\mathrm{ref}}^{(t)}(\cdot\mid x)$ and $p_{\mathrm{pro}}^{(t)}(\cdot\mid x)$ denote the token distributions of $\Mref$ and $\Mpro$ at position $t$. The local divergence is
\begin{equation}
    d_t(x)
    =
    \frac{1}{2}
    \left[
        \mathrm{KL}(p_{\mathrm{ref}}^{(t)}\|p_{\mathrm{pro}}^{(t)})
        +
        \mathrm{KL}(p_{\mathrm{pro}}^{(t)}\|p_{\mathrm{ref}}^{(t)})
    \right].
    \label{eq:sym_kl}
\end{equation}
The symmetric form makes the score independent of the order of the two auxiliary models. The score is obtained by aggregating the local divergences over a set of aligned positions $\mathcal{T}(x)$.
\begin{equation}
    s(x)
    =
    \frac{1}{|\mathcal{T}(x)|}
    \sum_{t\in\mathcal{T}(x)} d_t(x).
    \label{eq:score}
\end{equation}
The choice of $\mathcal{T}(x)$ depends on the task structure. For generative tasks, the score is computed on a fixed decoding path shared by the two auxiliary models. For multiple-choice tasks, it is computed on prompt-side positions near the answer-option token.

Given validation forget and retain sets $D_f^{\mathrm{val}}$ and $D_r^{\mathrm{val}}$, we calibrate the routing threshold by supervised search.
\begin{equation}
    \tau^{\star}
    =
    \arg\max_{\tau\in\Gamma}
    \mathcal{J}(\tau,D_f^{\mathrm{val}},D_r^{\mathrm{val}}),
    \label{eq:threshold_search}
\end{equation}
where $\Gamma$ is the candidate threshold set and $\mathcal{J}$ is a calibration objective defined on the validation samples. During deployment, CBD routes each query according to~\eqref{eq:routing}. Inputs with $s(x)>\tau$ are answered by $\Mref$, while inputs with $s(x)\le\tau$ proceed to the target LLM $M$. The complete training and deployment workflow is summarized in Algorithm~\ref{alg:cbd} in Section~\ref{sec:subspace}, whose base variant is obtained by extracting $U_r$ and using the complement projection~\eqref{eq:retain_orth_proj} in the training loop.

\section{Discriminative Fisher Basis Extraction}
\label{sec:subspace}

Section~\ref{sec:method} introduced the retain basis $U_r$ as an initial approximation to the directions that should be protected so that the probe model stays close to the reference model on retained behavior. However, when forget and retain samples are highly similar, the retain basis can overlap with directions needed for forget-side separation, so a retain-only projection may become too conservative.

\begin{table}
\centering
\caption{Illustrative forget and retain query pairs with high structural similarity on ToFU. Blue shading marks shared wording, and orange shading marks differing entities or attributes.}
\label{tab:high_similarity_examples}
\setlength{\tabcolsep}{3pt}
\renewcommand{\arraystretch}{1.1}
\footnotesize
\begin{tabular}{@{}>{\centering\arraybackslash}m{0.08\columnwidth}>{\centering\arraybackslash}m{0.15\columnwidth}>{\raggedright\arraybackslash}m{0.71\columnwidth}@{}}
\toprule
\multicolumn{1}{c}{Pair} & \multicolumn{1}{c}{Split} & \multicolumn{1}{c}{Query} \\
\midrule
\multirow[c]{2}{*}[-1.6ex]{1} & Forget & \shared{What is the full name} \shared{of the LGBTQ+ author} \shared{born in Tehran, Iran} on \diffpart{11/26/1972}? \\
 & Retain & \shared{What is the full name} \shared{of the LGBTQ+ author} who was \shared{born in Tehran, Iran} on \diffpart{September 22, 1976}? \\
\midrule
\multirow[c]{2}{*}[-1.6ex]{2} & Forget & \shared{Has any of} \diffpart{Adib Jarrah's} \shared{works been adapted into} \shared{films or series}? \\
 & Retain & \shared{Have any of} \diffpart{Femi Oluwatoyin's} \shared{works been adapted into} \shared{films or series}? \\
\midrule
\multirow[c]{2}{*}[-1.6ex]{3} & Forget & \shared{How has} \diffpart{Hsiao Yun-Hwa's} \shared{identity as a member} \shared{of the LGBTQ+ community} \shared{influenced her} \diffpart{work}? \\
 & Retain & \shared{How has} \diffpart{Laila Amira al-Faisal's} \shared{identity as a member} \shared{of the LGBTQ+ community} \shared{influenced her} \diffpart{writing}? \\
\bottomrule
\end{tabular}
\end{table}

\begin{figure}
    \centering
    \includegraphics[width=.84\columnwidth]{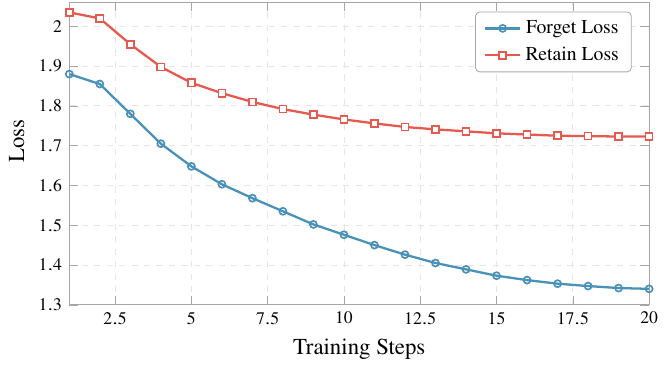}
    \caption{Coupled forget-side and retain-side loss trajectories during forget-side optimization. The co-movement indicates that the two subsets share update directions rather than occupying cleanly separable subspaces.}
    \label{fig:retain_forget_coupling}
\end{figure}

As Table~\ref{tab:high_similarity_examples} shows, forget and retain queries on ToFU often share almost the same question template and differ mainly in the author identity or one narrow attribute. This high overlap means that forget and retain samples can activate highly similar directions. Figure~\ref{fig:retain_forget_coupling} shows the resulting behavior. When the gradient-projection training rule is constructed from retain activation representations, the retain-side loss of $\Mpro$ still decreases together with the forget-side loss rather than staying nearly constant, which indicates that the two subsets are coupled through overlapping directions in the trainable subspace. Under this regime, a basis extracted only from retain activations tends to overlap with directions that also matter for forget-side separation. Projecting away all retain-associated directions would then suppress not only retain drift, but also part of the useful signal that makes $\Mpro$ deviate from $\Mref$ on forget inputs. Below we construct a basis that favors directions improving forget-side separation while keeping retain-side cost small.

\subsection{Discriminative Fisher Criterion}

We seek update directions that change the probe model more on forget inputs than on retain inputs. To connect this criterion with a computable local geometry, consider a small displacement $\Delta w$ around the reference point $w_0$ in the trainable parameter subspace. For a fixed input $x$, define $q(\cdot\mid x)=p_{w_0}(\cdot\mid x)$ and $p(\cdot\mid x)=p_{w_0+\Delta w}(\cdot\mid x)$. The local output change is measured by $\mathrm{KL}(q\|p)=\mathbb{E}_{y\sim q}[\log q(y\mid x)-\log p(y\mid x)]$. Applying a second-order Taylor expansion to $\log p_{w_0+\Delta w}(y\mid x)$ around $w_0$ gives
\[
\begin{aligned}
\log p_{w_0+\Delta w}(y\mid x)
&\approx
\log p_{w_0}(y\mid x) \\
&\quad
+g_x(y)^\top\Delta w
+\frac{1}{2}\Delta w^\top H_x(y)\Delta w,
\end{aligned}
\]
where $g_x(y)=\nabla_w\log p_w(y\mid x)\big|_{w_0}$ and $H_x(y)=\nabla_w^2\log p_w(y\mid x)\big|_{w_0}$ denote the per-sample gradient and Hessian of the log-likelihood. Substituting the expansion into the KL divergence gives
\[
\mathrm{KL}(q\|p)
\approx
-\mathbb{E}_{q}[g_x(y)]^\top\Delta w
-\frac{1}{2}\Delta w^\top \mathbb{E}_{q}[H_x(y)]\Delta w.
\]
The first-order term is zero because
\[
\mathbb{E}_{q}[g_x(y)]
=
\sum_y \nabla_w p_w(y\mid x)\big|_{w_0}
=
\nabla_w \sum_y p_w(y\mid x)\big|_{w_0}
=0.
\]
Under standard regularity conditions, the Fisher information matrix satisfies
\[
F(x)=\mathbb{E}_{q}[g_x(y)g_x(y)^\top]
=-\mathbb{E}_{q}[H_x(y)].
\]
Therefore,
\begin{equation}
    \mathrm{KL}\!\left(
        p_{w_0}(\cdot \mid x)\,\|\,p_{w_0+\Delta w}(\cdot \mid x)
    \right)
    \approx
    \frac{1}{2}\Delta w^{\top}F(x)\Delta w,
    \label{eq:local_kl}
\end{equation}
where $F(x)$ is the Fisher information matrix induced by input $x$ in the trainable parameter subspace~\cite{martens2020insights}. Averaging over the forget and retain sets yields
\begin{equation}
    F_f = \mathbb{E}_{x\sim D_f}[F(x)],
    \qquad
    F_r = \mathbb{E}_{x\sim D_r}[F(x)].
    \label{eq:fisher_avg}
\end{equation}

The desired update directions should produce large forget-side change while limiting retain-side drift. This requirement can be written as the constrained optimization problem
\begin{equation}
    \max_{\Delta w\neq 0}\ \Delta w^{\top}F_f\Delta w
    \quad
    \text{subject to}
    \quad
    \Delta w^{\top}F_r\Delta w\le \rho,
    \label{eq:rayleigh}
\end{equation}
where $\rho$ is the allowed retain-side budget. Since the retain Fisher matrix is estimated from finitely many samples and is therefore rank-deficient or ill-conditioned in practice, we damp the retain-side term and introduce
\begin{equation}
    C=F_r+\mu I, \qquad \mu>0,
    \label{eq:cmatrix}
\end{equation}
which keeps the constraint matrix positive definite. With this regularization, the direction-selection problem becomes
\begin{equation}
    \max_{\Delta w\neq 0}\ \Delta w^{\top}F_f\Delta w
    \quad
    \text{subject to}
    \quad
    \Delta w^{\top}C\Delta w\le \rho.
    \label{eq:rayleigh_reg}
\end{equation}

The regularized problem in~\eqref{eq:rayleigh_reg} can be read through its Lagrangian form
\begin{equation}
    \mathcal{L}_{\mathrm{lag}}(\Delta w,\lambda)
    =
    \Delta w^{\top}F_f\Delta w
    -
    \lambda\bigl(\Delta w^{\top}C\Delta w-\rho\bigr).
    \label{eq:lagrangian}
\end{equation}
Setting the derivative with respect to $\Delta w$ to zero gives
\begin{equation}
    F_f\Delta w = \lambda C\Delta w,
    \label{eq:gevp_raw}
\end{equation}
so every stationary direction is a generalized eigenvector of the matrix pair $(F_f,C)$. This derivation also clarifies the role of $\rho$. It controls the admissible retain-side displacement scale rather than the preferred direction itself. Since the basis construction only needs the generalized eigen-directions, Algorithm~\ref{alg:basis} does not treat $\rho$ as an independent hyperparameter. The eventual update magnitude is controlled later by the learning rate $\eta$ and the retain-side weight $\beta$ in Section~\ref{sec:method}.

Because the subspace used in training has dimension $k$ rather than one, we consider the natural multi-direction extension of~\eqref{eq:rayleigh_reg}, in which $k$ directions are selected jointly under a retain-side normalization,
\begin{equation}
    \max_{V\in\mathbb{R}^{d_w\times k},\ V^{\top}CV=I_k}\ \mathrm{tr}\!\left(V^{\top}F_fV\right).
    \label{eq:rayleigh_trace}
\end{equation}

\begin{theorem}\label{thm:gevp}
The objective in~\eqref{eq:rayleigh_trace} is maximized by the generalized eigenvectors $v_1,\dots,v_k$ associated with the $k$ largest generalized eigenvalues of the matrix pair $(F_f,C)$, i.e.,
\begin{equation}
    F_f v_i = \lambda_i Cv_i = \lambda_i(F_r+\mu I)v_i.
    \label{eq:gevp}
\end{equation}
\end{theorem}

\begin{proof}
Because $C\succ 0$, the substitution $\zeta_i=C^{1/2}v_i$ is invertible, and the constraint $V^{\top}CV=I_k$ becomes the requirement that $\Xi=[\zeta_1,\dots,\zeta_k]$ has orthonormal columns. The objective becomes
\[
\mathrm{tr}\!\left(\Xi^{\top}C^{-1/2}F_fC^{-1/2}\Xi\right).
\]
The matrix $C^{-1/2}F_fC^{-1/2}$ is symmetric positive semidefinite, and by the Ky Fan characterization the maximum of this trace over orthonormal $\Xi$ is attained by its eigenvectors associated with the $k$ largest eigenvalues. Writing $\zeta_i=C^{1/2}v_i$ gives
\[
C^{-1/2}F_fC^{-1/2}(C^{1/2}v_i)=\lambda_i(C^{1/2}v_i),
\]
which is equivalent to $F_fv_i=\lambda_i Cv_i$. For $k=1$ the same argument also solves~\eqref{eq:rayleigh_reg}, since the objective is homogeneous of degree two and $F_f$ is positive semidefinite with $F_f\neq 0$. The optimal value is then positive, the constraint is active at any maximizer, and the optimal direction coincides with the leading generalized eigenvector.
\end{proof}

Theorem~\ref{thm:gevp} shows that the leading generalized eigenvectors maximize forget-side change per unit retain-side cost. We stack the top-$k$ generalized eigenvectors as
\begin{equation}
    V_k=[v_1,\dots,v_k]
    \label{eq:basis_matrix}
\end{equation}
and define the discriminative basis $Q$ as the orthonormal factor of the QR factorization of $V_k$, so that $QQ^{\top}$ is an orthogonal projector onto $\mathrm{span}\{v_1,\dots,v_k\}$. This $Q$ is the basis used by the high-similarity variant of CBD. Its role is different from that of the retain basis $U_r$ in Section~\ref{sec:method}. The retain basis $U_r$ is an avoidance basis that marks directions which should not be disturbed because they are sensitive to retained behavior. The discriminative basis $Q$ is a selection basis that keeps directions which most improve forget-retain separation under controlled retain-side cost.

\subsection{Empirical Fisher Estimation and Efficient Computation}

The criterion above defines the optimization target of the final basis $Q$. We now realize it from sample gradients rather than explicit Hessian computation. For each supervised sample $(x_i,y_i)$ in a dataset $\mathcal{D}=\{(x_i,y_i)\}_{i=1}^{N}$, we evaluate the answer-side negative log-likelihood~\eqref{eq:nll} under teacher forcing at the common reference point $w_0$, and only the answer tokens contribute to this loss. For multiple-choice data, only the final answer-option token is supervised, which avoids tokenizer-dependent instability. The resulting sample gradient in the trainable auxiliary parameter space is
\begin{equation}
    g_i=\nabla_w \ell(w, x_i, y_i)\big|_{w_0}\in\mathbb{R}^{d_w}.
\end{equation}
The corresponding empirical Fisher matrix is estimated by the average outer product of these sample gradients,
\begin{equation}
    \widehat{F}=\frac{1}{N}\sum_{i=1}^{N} g_i g_i^{\top},
    \label{eq:emp_fisher_avg}
\end{equation}
which replaces the model expectation in the definition of $F(x)$ by the observed labels and is the standard empirical-Fisher surrogate~\cite{martens2020insights}. Stacking the gradients of the $N_f$ forget and $N_r$ retain samples used for basis extraction into matrices
\begin{equation}
    G_f=[g_{f,1},\dots,g_{f,N_f}],
    \qquad
    G_r=[g_{r,1},\dots,g_{r,N_r}],
\end{equation}
we can write the empirical Fisher matrices compactly as
\begin{equation}
    \widehat{F}_f=\frac{1}{N_f}G_fG_f^{\top},
    \qquad
    \widehat{F}_r=\frac{1}{N_r}G_rG_r^{\top}.
    \label{eq:emp_fisher}
\end{equation}
For basis extraction only, we use the same answer-token loss on both subsets so that forget and retain directions are measured in a common local geometry, while the probe-training objective in Section~\ref{sec:method} still keeps the retain-side KL term for behavioral alignment during optimization. The estimator is appropriate here because basis extraction is performed in the same low-rank auxiliary parameter space used by the probe model, and the forget and retain samples are both evaluated around the same reference point $w_0$.

Solving the generalized eigenvalue problem~\eqref{eq:gevp} directly in $\mathbb{R}^{d_w}$ is computationally expensive when $d_w$ is large, so we convert the problem from parameter scale to sample scale. For implementation, write
\begin{equation}
    \widehat{C}
    =
    \widehat{F}_r+\mu I
    =
    \mu I+\frac{1}{N_r}G_rG_r^{\top},
    \label{eq:chat_matrix}
\end{equation}
where the damping factor $\mu$ is set in proportion to the scale of $\widehat{F}_r$ by default. The empirical generalized eigenvalue problem is then
\begin{equation}
    \widehat{F}_f v=\lambda \widehat{C}v.
    \label{eq:empirical_gevp}
\end{equation}
Because $\widehat{F}_f=\frac{1}{N_f}G_fG_f^\top$, the forget-side term acts only through the span of the forget gradient matrix $G_f$, which allows the problem to be reduced to sample scale. Define
\begin{equation}
    Z=\widehat{C}^{-1}G_f\in\mathbb{R}^{d_w\times N_f},
\end{equation}
\begin{equation}
    \Phi=\frac{1}{N_f}G_f^{\top}Z=\frac{1}{N_f}G_f^{\top}\widehat{C}^{-1}G_f
    \in\mathbb{R}^{N_f\times N_f},
    \label{eq:small_matrix}
\end{equation}
\begin{equation}
    \Phi u=\lambda u.
    \label{eq:small_eig}
\end{equation}
If~\eqref{eq:small_eig} holds, then the corresponding parameter-space direction is recovered by
\begin{equation}
    v=Zu=\widehat{C}^{-1}G_fu.
    \label{eq:recover_v}
\end{equation}
Substituting~\eqref{eq:recover_v} into~\eqref{eq:empirical_gevp} recovers the same generalized eigenvalue relation. Conversely, since $\mathrm{rank}(\widehat{F}_f)\le N_f$, every generalized eigenvector with nonzero eigenvalue lies in the column span of $\widehat{C}^{-1}G_f$, so the large problem in $\mathbb{R}^{d_w}$ and the small problem in $\mathbb{R}^{N_f}$ share the same nonzero spectrum.

It remains to compute $Z=\widehat{C}^{-1}G_f$ without explicitly forming $\widehat{C}^{-1}$. Applying the Woodbury identity to~\eqref{eq:chat_matrix} gives
\begin{equation}
    \widehat{C}^{-1}
    =
    \frac{1}{\mu}I
    -
    \frac{1}{\mu^2 N_r}
    G_r
    S^{-1}
    G_r^{\top},
    \label{eq:woodbury}
\end{equation}
where only the retain-side sample matrix
\begin{equation}
    S=I+\frac{1}{\mu N_r}G_r^{\top}G_r
    \in\mathbb{R}^{N_r\times N_r}
    \label{eq:woodbury_small}
\end{equation}
has to be inverted. To make the computation explicit, define the cross matrix
\begin{equation}
    P=G_r^{\top}G_f
    \in\mathbb{R}^{N_r\times N_f}
    \label{eq:woodbury_cross}
\end{equation}
and solve the small linear system
\begin{equation}
    S\Psi=P,
    \qquad
    \Psi\in\mathbb{R}^{N_r\times N_f},
    \label{eq:woodbury_linear}
\end{equation}
typically by first computing the Cholesky factorization $S=LL^\top$ and then applying forward and backward substitution. Substituting the solution of~\eqref{eq:woodbury_linear} into~\eqref{eq:woodbury} yields
\begin{equation}
    Z
    =
    \widehat{C}^{-1}G_f
    =
    \frac{1}{\mu}G_f
    -
    \frac{1}{\mu^2 N_r}
    G_r\Psi.
    \label{eq:z_formula}
\end{equation}
Once $Z$ is available, the sample-scale matrix $\Phi$ follows directly from~\eqref{eq:small_matrix}. We then compute the top-$k$ eigenpairs of $\Phi$, recover the parameter-space directions through~\eqref{eq:recover_v}, and orthonormalize them into the final basis $Q$ as defined after Theorem~\ref{thm:gevp}.

This reformulation avoids dense inversion in the trainable parameter space. Its dominant cost comes from sample-gradient computation, the formation of the retain-side sample matrix $S$ and the solution of the linear system $S\Psi=P$, and the eigendecomposition of the sample-scale matrix $\Phi$. Since $N_f,N_r\ll d_w$ in typical settings, the cost scales with the sample counts rather than with $d_w$, avoiding the formation of full $d_w\times d_w$ Fisher matrices. Algorithm~\ref{alg:basis} gives the complete basis-extraction procedure used in CBD.

\begin{algorithm}
\caption{Discriminative Fisher Basis Extraction in CBD}
\label{alg:basis}
\begin{algorithmic}[1]
\Require Reference model $\Mref$, forget set $D_f$, retain set $D_r$, damping factor $\mu$, basis dimension $k$
\Ensure Basis matrix $Q$
\State Compute sample gradients at $w_0$ under teacher forcing on answer tokens and form $G_f,G_r$
\State Form $S=I+\frac{1}{\mu N_r}G_r^\top G_r$ and $P=G_r^\top G_f$
\State Compute the Cholesky factorization $S=LL^\top$
\State Solve $S\Psi=P$ by forward and backward substitution
\State Compute $Z=\frac{1}{\mu}G_f-\frac{1}{\mu^2 N_r}G_r\Psi$
\State Form $\Phi=\frac{1}{N_f}G_f^\top Z$ and compute its top-$k$ eigenpairs $(u_i,\lambda_i)$
\State Recover $v_i=Zu_i$ for $i=1,\dots,k$
\State Stack $V_k=[v_1,\dots,v_k]$ and set $Q$ to the orthonormal factor of its QR factorization
\State \Return $Q$
\end{algorithmic}
\end{algorithm}

Once the discriminative basis $Q$ has been extracted, the high-similarity variant of CBD replaces the retain-only complement projection~\eqref{eq:retain_orth_proj} with a projection onto the selected discriminative subspace. The forget-side projected gradient becomes
\begin{equation}
    \bar{g}_f = Q Q^{\top} g_f,
    \label{eq:proj_forget}
\end{equation}
where $QQ^{\top}$ is the projection matrix onto the final discriminative basis. The corresponding training direction is
\begin{equation}
    g = \bar{g}_f + \beta g_r,
    \label{eq:combined_grad}
\end{equation}
and the probe model is updated by
\begin{equation}
    w \leftarrow w - \eta g.
    \label{eq:param_update}
\end{equation}
This update no longer treats all retain-associated directions as forbidden. Instead, it keeps the directions that maximize forget-side change under controlled retain-side cost. Algorithm~\ref{alg:cbd} summarizes the complete CBD workflow.

\begin{algorithm}
\caption{CBD Training and Deployment}
\label{alg:cbd}
\begin{algorithmic}[1]
\Require Target model $M$, reference model $\Mref$, forget set $D_f$, retain set $D_r$, validation sets $D_f^{\mathrm{val}}, D_r^{\mathrm{val}}$, threshold candidates $\Gamma$, retain-side weight $\beta$, learning rate $\eta$, damping factor $\mu$, basis dimension $k$
\Ensure Probe model $\Mpro$, routing threshold $\tau$
\State Extract the discriminative basis $Q$ using Algorithm~\ref{alg:basis}
\State Initialize probe-model parameters with $w\leftarrow w_0$
\While{the fixed training-step budget is not reached}
    \State Sample mini-batches from $D_f$ and $D_r$
    \State Compute forget-side gradient $g_f$ and retain-side gradient $g_r$
    \State Project forget-side gradient onto the discriminative basis as $\bar{g}_f=QQ^{\top} g_f$
    \State Form the combined gradient $g=\bar{g}_f+\beta g_r$
    \State Update $w\leftarrow w-\eta g$
\EndWhile
\State Denote the resulting probe model by $\Mpro$
\State Compute validation scores using symmetric KL divergence between $\Mref$ and $\Mpro$
\State Select the routing threshold $\tau$ by supervised search
\State Deploy CBD according to~\eqref{eq:routing}
\State \Return $\Mpro$, $\tau$
\end{algorithmic}
\end{algorithm}

\section{Experiments}
\label{sec:exp}

\subsection{Experimental Setup}

All experiments were conducted on a server with four NVIDIA GeForce RTX 3090 GPUs (24 GB each), using Python 3.10 and PyTorch 2.1.1 with CUDA 11.8.

\textbf{Models and Datasets.} We evaluate CBD on ToFU~\cite{maini2024tofu}\footnote{\url{https://huggingface.co/datasets/locuslab/TOFU}} and Weapons of Mass Destruction Proxy (WMDP)~\cite{li2024wmdp}\footnote{\url{https://huggingface.co/datasets/cais/wmdp}}. ToFU contains 4,000 fictitious-author question-answer pairs, 100 real-author questions, and 117 world-fact questions. ToFU defines three unlearning settings, namely \textit{forget01}, \textit{forget05}, and \textit{forget10}, which remove 1\%, 5\%, and 10\% of the fictitious-author data, respectively, with the corresponding forget and retain split sizes listed in Table~\ref{tab:dataset-summary}. Perturbed answers are used only for the truth-ratio metric. We use the ToFU-released Llama-2-7B-Chat model\footnote{\url{https://huggingface.co/locuslab/tofu_ft_llama2-7b}} fine-tuned on the full ToFU question-answer set as the target model~\cite{maini2024tofu,touvron2023llama}.

\begin{table}
\centering
\caption{Datasets, target models, and main optimization hyperparameters. For WMDP, the listed sets are the unlearning and retained-utility evaluation sets. All LoRA-based runs use rank $32$, scaling factor $64$, and dropout $0.05$, while ToFU white-box uses full fine-tuning.}
\label{tab:dataset-summary}
\footnotesize
\renewcommand{\arraystretch}{1.1}
\setlength{\tabcolsep}{1.5pt}
\begin{tabular}{@{}lccc@{}}
\hline\hline
Dataset & Forget data & Retain data & Target model \\
\hline
ToFU\footnotemark[1] & 40/200/400 QA & 3,960/3,800/3,600 QA & Llama-2-7B-Chat\footnotemark[3] \\
WMDP\footnotemark[2] & 3,668 QA & 14,042 QA & Zephyr-7B-beta\footnotemark[4] \\
\hline\hline
\end{tabular}

\vspace{4pt}
\label{tab:hyperparams}
\setlength{\tabcolsep}{6pt}
\begin{tabular}{@{}llccc@{}}
\hline\hline
Dataset & Method group & LR & Batch & Accum. \\
\hline
\multirow{2}{*}{ToFU}
& White-box & $1\times10^{-5}$ & 4 & 8 \\
& Gray/black-box & $1.5\times10^{-4}$ & 4 & 2 \\
\hline
\multirow{2}{*}{WMDP}
& White-box & $2\times10^{-5}$ & 1 & 2 \\
& Gray/black-box & $2\times10^{-4}$ & 2 & 4 \\
\hline\hline
\end{tabular}
\end{table}

WMDP evaluates whether a model can answer hazardous-knowledge questions. Its unlearning evaluation set contains 3,668 four-choice questions, including 1,273 biology questions, 1,987 cybersecurity questions, and 408 chemistry questions. Its retained-utility evaluation uses Massive Multitask Language Understanding (MMLU)~\cite{hendrycks2021mmlu}, which contains 14,042 test questions over 57 subjects in the all-category setting. Following the standard WMDP setting, we use Zephyr-7B-beta\footnote{\url{https://huggingface.co/HuggingFaceH4/zephyr-7b-beta}} as the target model~\cite{li2024wmdp,tunstall2023zephyr}. On both benchmarks, CBD initializes the two auxiliary models from TinyLlama-1.1B-Chat-v1.0\footnote{\url{https://huggingface.co/TinyLlama/TinyLlama-1.1B-Chat-v1.0}}~\cite{zhang2024tinyllama}. For WMDP, probe training, basis extraction, and threshold calibration use 600 forget-side samples from the three hazardous domains and 1,200 retain-side samples, while the full question sets are used for final evaluation.

\textbf{Baselines.} We compare CBD with reference models and representative methods under different access assumptions. The reference models include 1) \textit{Target LLM}, the original target model before unlearning, and 2) \textit{Retrain LLM}, the model trained without the forget data. The Retrain LLM is reported only on ToFU, because WMDP does not provide a retrained reference.

The white-box baselines are the three objective families defined in Section~\ref{sec:preliminaries}, namely \textit{GA}~\cite{jang2023knowledge,yao2024machine}, \textit{DPO}~\cite{rafailov2023direct}, and \textit{NPO}~\cite{zhang2024negative}, each reported with its base objective and its retain-constrained \textit{+GD} and \textit{+KL} variants~\cite{yao2024machine}. The gray-box baselines are \textit{ULD}~\cite{ji2024reversing} and \textit{Offset}~\cite{huang2024offset}, which correct the target logits as in~\eqref{eq:uld} and~\eqref{eq:offset} and therefore require token-level logits, whereas CBD uses only final responses.

\textbf{Hyperparameters.} The lower block of Table~\ref{tab:hyperparams} reports the optimization settings used in the main comparisons, separated by dataset and method group because ToFU and WMDP use different target models, evaluation data, and training budgets.

\textbf{Evaluation Metrics.} On ToFU, we report ROUGE-L recall (RG), answer probability (Pr), and truth ratio (TR) following the official evaluation protocol~\cite{maini2024tofu}. RG measures lexical overlap between the generated answer and the reference answer, and Pr is the length-normalized probability $P_{\hat M}(a\mid q)^{1/|a|}$ that the evaluated model $\hat M$ assigns to the reference answer $a$ of question $q$. TR compares the average length-normalized probability $\bar p_p(q)$ of perturbed answers with that of the paraphrased true answer $\tilde a$ through $r(q)=\bar p_p(q)\,/\,P_{\hat M}(\tilde a\mid q)^{1/|\tilde a|}$ and is reported with the official scaling
\begin{equation}
    \mathrm{TR}(q)=
    \begin{cases}
    \min\{r(q), r(q)^{-1}\}, & q\in D_f,\\
    \max\{0, 1-r(q)\}, & q\in D_r.
    \end{cases}
    \label{eq:eval_tr}
\end{equation}
On the forget subset, lower RG and Pr indicate less answer recovery, while higher TR means that the model no longer strongly prefers the paraphrased true answer over perturbed alternatives. On the retained data, real-author questions, and world-fact questions, higher RG, Pr, and TR indicate better retained utility. We summarize retained utility by Model Utility (MU), the harmonic mean of RG, Pr, and TR over these retained evaluations. Since an unlearning method should suppress forget data without damaging retained utility, we also report the Forget-Retain Trade-off (FRT) metric~\cite{wang2025erasing}. FRT normalizes retained utility by the remaining answer overlap and answer probability on the forget subset,
\begin{equation}
    \mathrm{FRT}=\frac{\mathrm{MU}}{(\mathrm{RG}_f+\mathrm{Pr}_f)/2}.
    \label{eq:eval_frt}
\end{equation}
Higher FRT indicates that the model preserves more retained utility for the same level of residual answer recovery on the forget subset. On WMDP, lower hazardous-knowledge accuracy indicates stronger unlearning, while higher MMLU accuracy indicates better general capability. Because CBD makes a query-level routing decision before invoking the target model, we also report routing accuracy, true-positive rate, and false-positive rate in percent, together with the AUC, where higher routing accuracy, true-positive rate, and AUC and lower false-positive rate are preferred. During evaluation, token probabilities are used only to score the local open-source models on these benchmarks, while CBD itself requires only final text responses at deployment.

\subsection{Performance Evaluation}

The evaluation has four parts. The first part compares CBD with white-box and gray-box baselines on ToFU and WMDP. The second part varies key hyperparameters to check whether the results remain stable beyond one selected setting. The third part examines the stability of each method across training steps. The fourth part compares CBD using the discriminative Fisher basis (DFB), denoted as CBD (with DFB), with CBD using GPM, denoted as CBD (with GPM), to isolate the effect of basis construction. Unless stated otherwise, CBD refers to CBD (with DFB).

\begin{table*}
\centering
\caption{Experimental results on the ToFU dataset for unlearning 1\% and 5\% of data. For compact presentation, each setting reports forget-side RG/Pr/TR and summary MU/FRT. Within each method group, the best value in each column is boldfaced.}
\setlength{\tabcolsep}{8pt}
\renewcommand{\arraystretch}{1.1}
\footnotesize
\begin{tabular}{@{}l|ccccc|ccccc@{}}
\hline\hline
\multirow{2}{*}{Method} & \multicolumn{5}{c|}{ToFU \textit{forget01}} & \multicolumn{5}{c}{ToFU \textit{forget05}} \\
\cline{2-6}\cline{7-11}
& RG$\downarrow$ & Pr$\downarrow$ & TR$\uparrow$ & MU$\uparrow$ & FRT$\uparrow$ & RG$\downarrow$ & Pr$\downarrow$ & TR$\uparrow$ & MU$\uparrow$ & FRT$\uparrow$ \\
\hline
\multicolumn{11}{c}{\rule{0pt}{2.4ex}\textit{Standard}} \\
\hline
\textbf{Target LLM} & 89.16 & 76.37 & 54.54 & 75.37 & 0.91 & 89.53 & 79.20 & 51.97 & 75.37 & 0.89 \\
\textbf{Retrain LLM} & 40.76 & 33.79 & 69.42 & 73.86 & 1.98 & 39.49 & 34.39 & 69.30 & 74.00 & 2.00 \\
\hline
\multicolumn{11}{c}{\rule{0pt}{2.4ex}\textit{White-box}} \\
\hline
\textbf{GA} & 42.86 & 39.32 & 54.01 & 64.45 & 1.57 & 31.96 & 52.84 & 53.76 & 44.96 & 1.06 \\
\textbf{GA+GD} & 26.01 & \groupbest{8.10} & 53.78 & 57.35 & \groupbest{3.36} & 37.00 & 49.62 & 51.54 & 51.23 & 1.18 \\
\textbf{GA+KL} & 37.17 & 26.86 & 53.40 & 63.93 & 2.00 & 41.14 & 61.96 & 49.71 & 52.79 & 1.02 \\
\textbf{DPO} & 29.46 & 63.14 & 59.33 & 67.06 & 1.45 & \groupbest{19.43} & 65.11 & 56.43 & 60.57 & 1.43 \\
\textbf{DPO+GD} & \groupbest{25.83} & 62.49 & 60.10 & 67.26 & 1.52 & 30.07 & 67.02 & 55.84 & 65.66 & 1.35 \\
\textbf{DPO+KL} & 27.78 & 63.55 & 59.26 & \groupbest{67.31} & 1.47 & 37.15 & 66.76 & 55.86 & \groupbest{66.34} & 1.28 \\
\textbf{NPO} & 34.25 & 34.46 & 59.91 & 55.95 & 1.63 & 35.55 & 65.53 & 53.44 & 49.28 & 0.98 \\
\textbf{NPO+GD} & 33.32 & 38.12 & 59.68 & 60.92 & 1.71 & 30.68 & 50.14 & 61.61 & 61.00 & 1.51 \\
\textbf{NPO+KL} & 37.66 & 44.67 & \groupbest{60.28} & 65.17 & 1.58 & 26.38 & \groupbest{45.40} & \groupbest{63.75} & 56.44 & \groupbest{1.57} \\
\hline
\multicolumn{11}{c}{\rule{0pt}{2.4ex}\textit{Gray-box}} \\
\hline
\textbf{ULD} & \groupbest{12.75} & \groupbest{21.93} & \groupbest{64.56} & \groupbest{76.08} & \groupbest{4.39} & \groupbest{17.88} & 72.93 & \groupbest{56.45} & \groupbest{64.44} & \groupbest{1.42} \\
\textbf{Offset} & 20.28 & 39.49 & 11.45 & 37.89 & 1.27 & 25.91 & \groupbest{69.34} & 25.38 & 35.06 & 0.74 \\
\hline
\multicolumn{11}{c}{\rule{0pt}{2.4ex}\textit{Black-box}} \\
\hline
\textbf{CBD} & \groupbest{34.36} & \groupbest{26.22} & \groupbest{79.84} & \groupbest{75.31} & \groupbest{2.49} & \groupbest{32.95} & \groupbest{25.59} & \groupbest{78.66} & \groupbest{74.76} & \groupbest{2.55} \\
\hline\hline
\end{tabular}
\label{tab:tofu_forget01_forget05_compact}
\end{table*}

\begin{table*}
\centering
\caption{Detailed experimental results on the ToFU dataset for unlearning 10\% of data. Within each method group, the best value in each column is boldfaced.}
\setlength{\tabcolsep}{5pt}
\renewcommand{\arraystretch}{1.1}
\footnotesize
\begin{tabular}{@{}l|ccc|ccc|ccc|ccc|cc@{}}
\hline\hline
\multirow{2}{*}{Method} & \multicolumn{3}{c|}{Forget data} & \multicolumn{3}{c|}{Retain data} & \multicolumn{3}{c|}{Real Authors} & \multicolumn{3}{c|}{World Facts} & \multicolumn{2}{c}{Summary} \\
\cline{2-15}
& RG$\downarrow$ & Pr$\downarrow$ & TR$\uparrow$ & RG$\uparrow$ & Pr$\uparrow$ & TR$\uparrow$ & RG$\uparrow$ & Pr$\uparrow$ & TR$\uparrow$ & RG$\uparrow$ & Pr$\uparrow$ & TR$\uparrow$ & MU$\uparrow$ & FRT$\uparrow$ \\
\hline
\multicolumn{15}{c}{\rule{0pt}{2.4ex}\textit{Standard}} \\
\hline
\textbf{Target LLM} & 91.85 & 78.40 & 51.90 & 91.37 & 96.31 & 94.30 & 92.30 & 58.04 & 74.05 & 89.17 & 51.15 & 66.02 & 75.37 & 0.89 \\
\textbf{Retrain LLM} & 39.90 & 32.78 & 69.18 & 87.66 & 94.19 & 92.97 & 92.05 & 56.44 & 72.65 & 90.03 & 48.17 & 62.15 & 73.12 & 2.01 \\
\hline
\multicolumn{15}{c}{\rule{0pt}{2.4ex}\textit{White-box}} \\
\hline
\textbf{GA} & 29.98 & 45.25 & 58.78 & 28.47 & 1.36 & 65.94 & 35.00 & 37.51 & 51.09 & 64.99 & 37.51 & 54.07 & 9.77 & 0.26 \\
\textbf{GA+GD} & 41.27 & 69.30 & 49.22 & 46.30 & 49.24 & \groupbest{97.54} & 75.90 & 38.12 & 54.08 & 85.40 & 38.67 & 53.47 & 54.11 & 0.98 \\
\textbf{GA+KL} & 37.04 & 70.65 & 50.72 & 37.17 & 26.45 & 94.75 & 71.03 & 33.90 & 46.79 & 84.33 & 38.34 & 52.75 & 45.66 & 0.85 \\
\textbf{DPO} & \groupbest{28.16} & 65.16 & 55.98 & 41.34 & 90.66 & 89.61 & 91.30 & 47.17 & 60.99 & 85.75 & 44.47 & 56.06 & 61.26 & 1.31 \\
\textbf{DPO+GD} & 41.59 & 67.23 & 55.51 & 64.30 & \groupbest{93.06} & 90.49 & \groupbest{92.30} & 46.77 & 60.14 & 86.61 & 44.31 & 56.03 & 65.14 & 1.20 \\
\textbf{DPO+KL} & 43.56 & 67.04 & 55.50 & \groupbest{66.07} & 93.01 & 90.39 & \groupbest{92.30} & 46.88 & 60.34 & 85.75 & 44.35 & 56.17 & \groupbest{65.36} & 1.18 \\
\textbf{NPO} & 28.94 & 44.68 & 63.55 & 29.07 & 26.89 & 80.31 & 87.97 & 45.54 & 59.57 & 87.04 & 43.76 & 56.33 & 48.32 & 1.31 \\
\textbf{NPO+GD} & 35.59 & 37.71 & 69.04 & 35.01 & 39.48 & 70.48 & 81.20 & \groupbest{53.13} & \groupbest{70.03} & \groupbest{88.11} & \groupbest{49.28} & \groupbest{63.34} & 55.93 & 1.53 \\
\textbf{NPO+KL} & 34.44 & \groupbest{35.91} & \groupbest{70.25} & 33.40 & 37.92 & 66.76 & 79.20 & 51.82 & 68.00 & 84.76 & 48.70 & 62.64 & 54.15 & \groupbest{1.54} \\
\hline
\multicolumn{15}{c}{\rule{0pt}{2.4ex}\textit{Gray-box}} \\
\hline
\textbf{ULD} & \groupbest{18.69} & 75.89 & \groupbest{54.85} & \groupbest{35.79} & \groupbest{60.73} & 94.84 & \groupbest{68.00} & \groupbest{57.89} & 72.66 & \groupbest{81.13} & \groupbest{57.45} & \groupbest{73.19} & \groupbest{62.46} & \groupbest{1.32} \\
\textbf{Offset} & 26.95 & \groupbest{75.33} & 27.35 & 34.71 & 9.88 & \groupbest{96.34} & 32.73 & 52.62 & \groupbest{74.54} & 55.73 & 48.24 & 72.73 & 35.18 & 0.69 \\
\hline
\multicolumn{15}{c}{\rule{0pt}{2.4ex}\textit{Black-box}} \\
\hline
\textbf{CBD} & \groupbest{38.19} & \groupbest{28.67} & \groupbest{76.50} & \groupbest{90.43} & \groupbest{94.63} & \groupbest{93.35} & \groupbest{92.30} & \groupbest{57.44} & \groupbest{73.18} & \groupbest{89.17} & \groupbest{51.15} & \groupbest{66.02} & \groupbest{74.90} & \groupbest{2.24} \\
\hline\hline
\end{tabular}
\label{tab:tofu10}
\end{table*}

\begin{table}
\centering
\caption{Experimental results on WMDP. Within each method group, the best value in each column is boldfaced.}
\setlength{\tabcolsep}{7.5pt}
\renewcommand{\arraystretch}{1.1}
\footnotesize
\begin{tabular}{@{}l|cccc|c@{}}
\hline\hline
Method & Overall$\downarrow$ & Bio$\downarrow$ & Cyber$\downarrow$ & Chem$\downarrow$ & MMLU$\uparrow$ \\
\hline
\multicolumn{6}{c}{\rule{0pt}{2.4ex}\textit{Standard}} \\
\hline
\textbf{Target LLM} & 50.95 & 65.28 & 42.17 & 49.02 & 59.01 \\
\hline
\multicolumn{6}{c}{\rule{0pt}{2.4ex}\textit{White-box}} \\
\hline
\textbf{GA} & 30.97 & 37.86 & 27.63 & 25.74 & 33.44 \\
\textbf{GA+GD} & 38.14 & 47.92 & 32.31 & 36.03 & 47.19 \\
\textbf{GA+KL} & \groupbest{26.44} & \groupbest{27.42} & \groupbest{26.07} & \groupbest{25.25} & 28.96 \\
\textbf{DPO} & 38.17 & 46.74 & 33.42 & 34.56 & 41.90 \\
\textbf{DPO+GD} & 46.56 & 60.49 & 37.80 & 45.83 & 55.21 \\
\textbf{DPO+KL} & 42.75 & 54.52 & 36.13 & 38.24 & 51.38 \\
\textbf{NPO} & 31.22 & 36.29 & 28.84 & 26.96 & 33.53 \\
\textbf{NPO+GD} & 42.39 & 54.67 & 34.52 & 42.40 & \groupbest{55.97} \\
\textbf{NPO+KL} & 31.00 & 32.44 & 31.05 & 26.23 & 32.31 \\
\hline
\multicolumn{6}{c}{\rule{0pt}{2.4ex}\textit{Gray-box}} \\
\hline
\textbf{ULD} & \groupbest{49.35} & \groupbest{65.51} & \groupbest{39.10} & \groupbest{48.77} & \groupbest{58.96} \\
\textbf{Offset} & 50.63 & 65.67 & 41.32 & 49.02 & 58.20 \\
\hline
\multicolumn{6}{c}{\rule{0pt}{2.4ex}\textit{Black-box}} \\
\hline
\textbf{CBD} & \groupbest{25.68} & \groupbest{25.77} & \groupbest{25.36} & \groupbest{26.96} & \groupbest{52.67} \\
\hline\hline
\end{tabular}
\label{tab:wmdp-main}
\end{table}

\begin{figure*}[!b]
\centering
\includegraphics[width=0.84\textwidth]{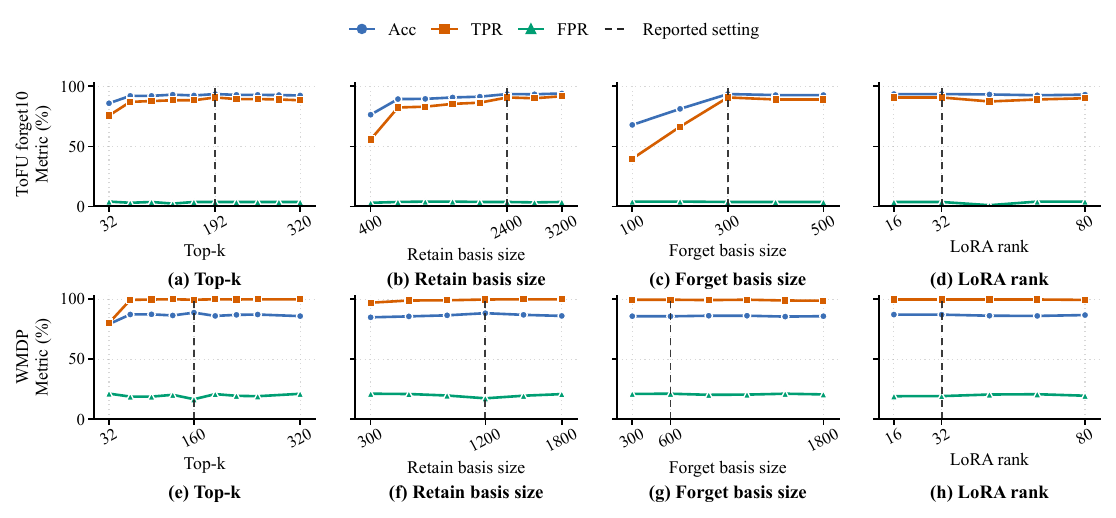}
\caption{Hyperparameter sensitivity of CBD on ToFU \textit{forget10} (top row) and WMDP (bottom row). The four columns vary the top-$k$ value, the retain basis size, the forget basis size, and the LoRA rank, respectively. The dashed line marks the configuration used in the main experiments.}
\label{fig:hparam}
\end{figure*}

\textbf{Performance Comparison with Existing Methods.} Table~\ref{tab:tofu_forget01_forget05_compact} compactly reports the forget-side and summary results under the \textit{forget01} and \textit{forget05} settings, while Table~\ref{tab:tofu10} provides the detailed results under the more challenging \textit{forget10} setting. The white-box methods trade off reducing answer recovery against preserving retained utility, and the trade-off sharpens as the unlearning ratio grows. GA+GD attains the best white-box FRT on \textit{forget01} at 3.36 but only at a reduced MU of 57.35, and on \textit{forget10} GA collapses to an MU of 9.77 with an FRT of 0.26. The DPO variants preserve more utility but leave forget-set answer probabilities well above the retrained reference, and the NPO variants stay closer to that reference but still above it.

The gray-box methods behave unevenly across the three settings. ULD is the strongest method on \textit{forget01}, where its FRT of 4.39 exceeds the 2.49 of CBD. However, ULD requires target-model logits that are unavailable in the API-only setting, and its advantage disappears on the larger splits, where its FRT falls to 1.42 and 1.32. Offset does not achieve consistent suppression, and its MU drops sharply on all three splits. CBD stays the most consistent across splits, keeping forget-set Pr between 25.59 and 28.67, forget-set TR above 76, and MU between 74.76 and 75.31 across all three splits. On \textit{forget10}, Table~\ref{tab:tofu10} further shows that the retained-side metrics of CBD stay close to the target model because retained queries that are not rerouted reach the unchanged target LLM, so the retained-utility cost of CBD comes only from false-positive routing.

Table~\ref{tab:wmdp-main} reports the WMDP results, where the same trade-off appears. GA+KL achieves the strongest white-box suppression with an overall hazardous accuracy of 26.44, but its MMLU drops to 28.96, while DPO+GD and NPO+GD preserve MMLU above 55 yet barely suppress hazardous accuracy. The gray-box baselines keep MMLU near the target LLM but achieve almost no suppression, with 49.35 for ULD and 50.63 for Offset. CBD reduces the overall hazardous accuracy to 25.68, close to the four-choice random-guess level, while preserving an MMLU accuracy of 52.67, giving the lowest hazardous accuracy among all compared methods at a smaller general-capability cost than GA+KL.

\textbf{Hyperparameter Sensitivity.} Figure~\ref{fig:hparam} varies the top-$k$ value, the retain basis size, the forget basis size, and the LoRA rank on ToFU \textit{forget10} and WMDP. On both datasets, the top-$k$ value has the largest effect, since a too-small value discards useful separation directions and lowers routing accuracy and true-positive rate, whereas beyond the main configuration the curves change only slightly and the false-positive rate stays low. The retain basis behaves similarly, as a small basis represents retained behavior incompletely and weakens routing. The forget basis matters on ToFU, where 100 or 200 forget-set samples are insufficient but 300 samples stabilize the curves. The LoRA rank has the smallest effect, since the routing accuracy and true-positive rate stay almost unchanged as it varies from 16 to 80.

\begin{figure*}[!t]
\centering
\includegraphics[width=0.86\textwidth]{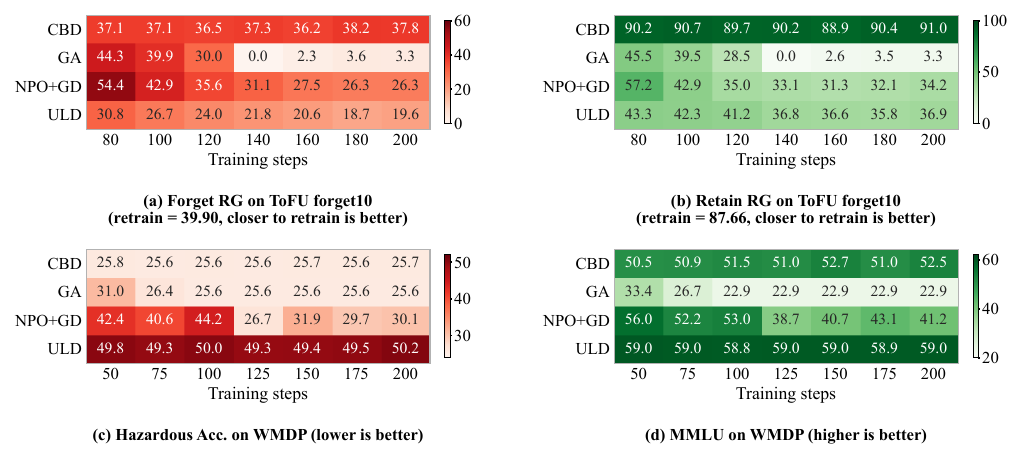}
\caption{Cross-method step stability on ToFU \textit{forget10} and WMDP. The heatmaps illustrate the stability of unlearning utility across training steps for different methods.}
\label{fig:step_stability}
\end{figure*}

\textbf{Cross-Method Step Stability.} Figure~\ref{fig:step_stability} reports unlearning metrics as functions of training steps for CBD, GA, NPO+GD, and ULD. On ToFU \textit{forget10}, we report forget-set and retain-set RG against the retrained reference values of 39.90 and 87.66. On WMDP, we report overall hazardous accuracy with a random-guess level of 25 and MMLU with a target value of 59.01.
CBD varies little across training steps on both datasets, with forget-set RG within $[36.23, 38.19]$ and retain-set RG around 90 on ToFU \textit{forget10}, and hazardous accuracy at 25.6--25.8 with MMLU at 50.5--52.7 on WMDP. In contrast, GA collapses, with both RG metrics dropping to 0.00 by step 140 on ToFU \textit{forget10} and MMLU falling to the random-guess level by step 100 on WMDP. NPO+GD degrades substantially as its retain-set RG falls from 57.17 to 34.21 between steps 80 and 200, and ULD remains flat but achieves little unlearning on WMDP, staying near 49.3 hazardous accuracy. Direct weight editing thus depends heavily on the training budget, while CBD does not.

\begin{figure}[!t]
\centering
\includegraphics[width=.8\columnwidth]{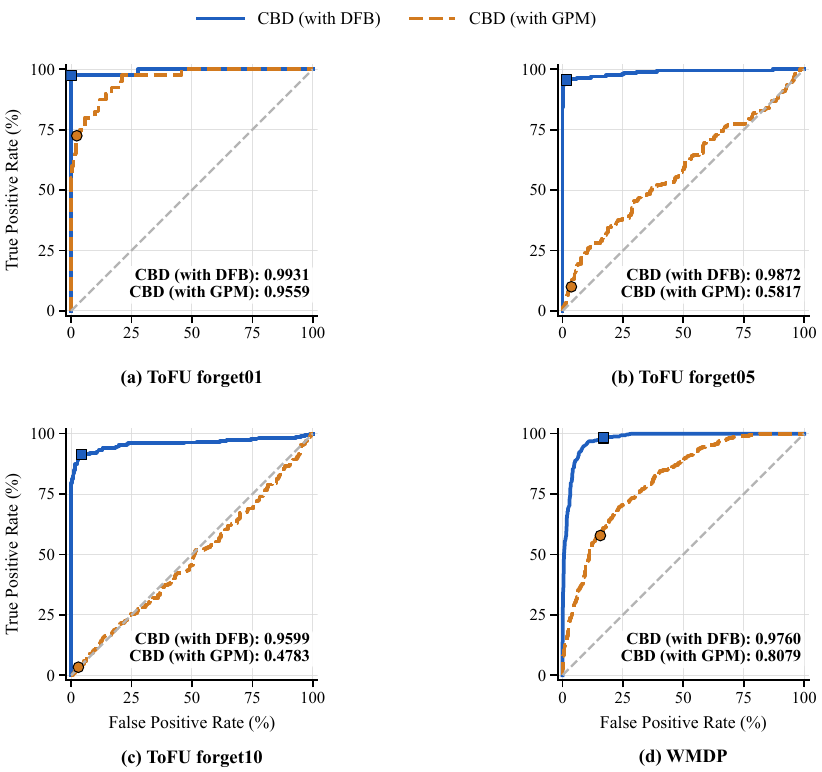}
\caption{Receiver operating characteristic curves comparing CBD (with DFB) and CBD (with GPM) across the three ToFU splits and WMDP, obtained by sweeping the routing threshold on the unlearning-relevance score. Square and circle markers denote the operating thresholds of CBD (with DFB) and CBD (with GPM), respectively.}
\label{fig:roc}
\end{figure}

\textbf{Routing Quality Comparison.} We isolate the effect of basis construction by comparing CBD (with DFB) and CBD (with GPM) under the same training budget. Figure~\ref{fig:roc} reports the resulting ROC curves, where a higher curve means more forget queries are correctly routed away at the same false-positive rate on retained queries. On ToFU, CBD (with DFB) attains AUC values of 0.9931, 0.9872, and 0.9599 on \textit{forget01}, \textit{forget05}, and \textit{forget10}, whereas CBD (with GPM) reaches 0.9559, 0.5817, and 0.4783. The AUC gap is about 0.04 on \textit{forget01} but 0.48 on \textit{forget10}, as the growing forget set covers more shared templates and tightens forget-retain coupling. At the operating threshold on \textit{forget10}, CBD (with GPM) detects almost no forget queries with a true-positive rate of 3.33, so most forget queries still reach the target model and routing fails to block them on the forget side. CBD (with DFB) instead reaches 95.00 routing accuracy with a true-positive rate of 93.00 and a false-positive rate of 3.00.

WMDP shows a similar result, where CBD (with DFB) achieves an AUC of 0.9760 and 88.72 routing accuracy at the operating threshold, compared with 0.8079 and 70.11 for CBD (with GPM). The discriminative Fisher basis therefore routes more accurately in exactly the cases where forget and retain queries are hard to separate. Across all experiments, the retained-utility cost of CBD is confined to false-positive routing, and each query adds only the cost of scoring by the two small auxiliary models before reaching the target LLM. Probe training and basis extraction take 152 to 185 seconds with a peak memory of 15.4 to 22.4 GB, which fits within a single 24 GB GPU.

\section{Conclusion}
\label{sec:conclusion}

This paper presents CBD, a black-box unlearning framework for API-only LLM services. CBD routes unlearning-related queries away from the target LLM using the behavioral divergence between a frozen reference model and a trained probe model, without editing the target or accessing its logits. A gradient-statistics-based discriminative basis improves this routing when forget and retain data are highly similar. Across ToFU and WMDP, CBD attains a better unlearning-utility trade-off than eleven baselines, improving the forget-retain trade-off by about 45\% over the second-best method on ToFU forget10 and lowering WMDP hazardous-knowledge accuracy to near the random-guess level while preserving MMLU.

\bibliographystyle{IEEEtran}
\bibliography{ref}

@article{qu2025mobile,
  title={Mobile Edge Intelligence for Large Language Models: A Contemporary Survey},
  author={Qu, Guanqiao and Chen, Qiyuan and Wei, Wei and Lin, Zheng and Chen, Xianhao and Huang, Kaibin},
  journal={IEEE Communications Surveys \& Tutorials},
  volume={27},
  pages={3820--3860},
  year={2025}
}

@inproceedings{carlini2021extracting,
  title={Extracting Training Data from Large Language Models},
  author={Carlini, Nicholas and others},
  booktitle={30th USENIX security symposium (USENIX Security 21)},
  pages={2633--2650},
  year={2021}
}

@inproceedings{yao2024machine,
  title={Machine unlearning of pre-trained large language models},
  author={Yao, Jin and others},
  booktitle={Proceedings of the 62nd Annual Meeting of the Association for Computational Linguistics (Volume 1: Long Papers)},
  pages={8403--8419},
  year={2024},
}

@article{liu2025rethinking,
  title={Rethinking Machine Unlearning for Large Language Models},
  author={Liu, Sijia and others},
  journal={Nature Machine Intelligence},
  volume={7},
  pages={181--194},
  year={2025}
}

@article{yao2024large,
  title={Large Language Model Unlearning},
  author={Yao, Yuanshun and Xu, Xiaojun and Liu, Yang},
  journal={Advances in Neural Information Processing Systems},
  volume={37},
  pages={105425--105475},
  year={2024},
}

@inproceedings{zhang2024negative,
  title={Negative preference optimization: From catastrophic collapse to effective unlearning},
  author={Zhang, Ruiqi and Lin, Licong and Bai, Yu and Mei, Song},
  booktitle={Proceedings of the First Conference on Language Modeling},
  year={2024},
}

@article{ji2024reversing,
  title={Reversing the Forget-Retain Objectives: An Efficient {LLM} Unlearning Framework from Logit Difference},
  author={Ji, Jiabao and others},
  journal={Advances in Neural Information Processing Systems},
  volume={37},
  pages={12581--12611},
  year={2024},
}

@article{huang2024offset,
  title={Offset Unlearning for Large Language Models},
  author={Huang, James Y. and others},
  journal={Transactions on Machine Learning Research},
  year={2025},
}

@inproceedings{bourtoule2021machine,
  title={Machine unlearning},
  author={Bourtoule, Lucas and others},
  booktitle={2021 IEEE symposium on security and privacy (SP)},
  pages={141--159},
  year={2021},
}

@article{rafailov2023direct,
  title={Direct Preference Optimization: Your Language Model is Secretly a Reward Model},
  author={Rafailov, Rafael and Sharma, Archit and Mitchell, Eric and Manning, Christopher D. and Ermon, Stefano and Finn, Chelsea},
  journal={Advances in Neural Information Processing Systems},
  volume={36},
  pages={53728--53741},
  year={2023}
}

@inproceedings{wang2024llm,
  title={{LLM} Unlearning via Loss Adjustment with Only Forget Data},
  author={Wang, Yaxuan and others},
  booktitle={International Conference on Learning Representations},
  year={2025},
}

@inproceedings{maini2024tofu,
  title={{TOFU}: A Task of Fictitious Unlearning for {LLM}s},
  author={Maini, Pratyush and Feng, Zhili and Schwarzschild, Avi and Lipton, Zachary C and Kolter, J Zico},
  booktitle={First Conference on Language Modeling},
  year={2024},
}

@inproceedings{hu2022lora,
  title={{LoRA}: Low-Rank Adaptation of Large Language Models},
  author={Hu, Edward J. and others},
  booktitle={International Conference on Learning Representations},
  year={2022}
}

@article{martens2020insights,
  title={New Insights and Perspectives on the Natural Gradient Method},
  author={Martens, James},
  journal={Journal of Machine Learning Research},
  volume={21},
  number={146},
  pages={1--76},
  year={2020}
}

@inproceedings{jang2023knowledge,
  title={Knowledge Unlearning for Mitigating Privacy Risks in Language Models},
  author={Jang, Joel and others},
  booktitle={Proceedings of the 61st Annual Meeting of the Association for Computational Linguistics (Volume 1: Long Papers)},
  pages={14389--14408},
  year={2023},
}

@inproceedings{chen2023forget,
  title={Unlearn What You Want to Forget: Efficient Unlearning for {LLM}s},
  author={Chen, Jiaao and Yang, Diyi},
  booktitle={Proceedings of the 2023 Conference on Empirical Methods in Natural Language Processing},
  pages={12041--12052},
  year={2023},
}

@inproceedings{kassem2023dememorization,
  title={Preserving Privacy through Dememorization: An Unlearning Technique for Mitigating Memorization Risks in Language Models},
  author={Kassem, Aly and Mahmoud, Omar and Saad, Sherif},
  booktitle={Proceedings of the 2023 Conference on Empirical Methods in Natural Language Processing},
  pages={4360--4379},
  year={2023}
}

@inproceedings{ren2025framework,
  title={A General Framework to Enhance Fine-tuning-based {LLM} Unlearning},
  author={Ren, Jie and others},
  booktitle={Findings of the Association for Computational Linguistics: ACL 2025},
  pages={18464--18476},
  year={2025}
}

@inproceedings{vasilev2025unilogit,
  title={Unilogit: Robust Machine Unlearning for {LLM}s Using Uniform-Target Self-Distillation},
  author={Vasilev, Stefan and Herold, Christian and Liao, Baohao and Hashemi, Seyyed Hadi and Khadivi, Shahram and Monz, Christof},
  booktitle={Findings of the Association for Computational Linguistics: ACL 2025},
  pages={22453--22472},
  year={2025},
}

@inproceedings{takashiro2025answer,
  title={Answer When Needed, Forget When Not: Language Models Pretend to Forget via In-Context Knowledge Unlearning},
  author={Takashiro, Shota and Kojima, Takeshi and Gambardella, Andrew and Cao, Qi and Iwasawa, Yusuke and Matsuo, Yutaka},
  booktitle={Findings of the Association for Computational Linguistics: ACL 2025},
  pages={24872--24885},
  year={2025},
}

@inproceedings{gu2025concealment,
  title={From Evasion to Concealment: Stealthy Knowledge Unlearning for {LLM}s},
  author={Gu, Tianle and others},
  booktitle={Findings of the Association for Computational Linguistics: ACL 2025},
  pages={10261--10279},
  year={2025}
}

@inproceedings{niwa2025belief,
  title={Rectifying Belief Space via Unlearning to Harness {LLM}s' Reasoning},
  author={Niwa, Ayana and Kaneko, Masahiro and Inui, Kentaro},
  booktitle={Findings of the Association for Computational Linguistics: ACL 2025},
  pages={25060--25075},
  year={2025}
}

@inproceedings{dou2025copyright,
  title={Avoiding Copyright Infringement via Large Language Model Unlearning},
  author={Dou, Guangyao and Liu, Zheyuan and Lyu, Qing and Ding, Kaize and Wong, Eric},
  booktitle={Findings of the Association for Computational Linguistics: NAACL 2025},
  pages={5191--5215},
  year={2025},
}

@inproceedings{wan2025token,
  title={Not Every Token Needs Forgetting: Selective Unlearning Balancing Forgetting and Utility in Large Language Models},
  author={Wan, Yixin and Ramakrishna, Anil and Chang, Kai-Wei and Cevher, Volkan and Gupta, Rahul},
  booktitle={Findings of the Association for Computational Linguistics: EMNLP 2025},
  pages={1827--1835},
  year={2025},
}

@inproceedings{xie2025reveal,
  title={Reveal and Release: Iterative {LLM} Unlearning with Self-Generated Data},
  author={Xie, Linxi and Teng, Xin and Ke, Shichang and Wen, Hongyi and Wan, Shenji},
  booktitle={Findings of the Association for Computational Linguistics: EMNLP 2025},
  pages={23887--23899},
  year={2025},
}

@inproceedings{tian2024forget,
  title={To Forget or Not? Towards Practical Knowledge Unlearning for Large Language Models},
  author={Tian, Bozhong and others},
  booktitle={Findings of the Association for Computational Linguistics: EMNLP 2024},
  pages={1524--1537},
  year={2024},
}

@article{liu2024eco,
  title={Large Language Model Unlearning via Embedding-Corrupted Prompts},
  author={Liu, Chris Yuhao and Wang, Yaxuan and Flanigan, Jeffrey and Liu, Yang},
  journal={Advances in Neural Information Processing Systems},
  volume={37},
  year={2024},
}

@inproceedings{patil2024sensitive,
  title={Can Sensitive Information Be Deleted From {LLM}s? Objectives for Defending Against Extraction Attacks},
  author={Patil, Vaidehi and Hase, Peter and Bansal, Mohit},
  booktitle={International Conference on Learning Representations},
  year={2024}
}

@inproceedings{jia2024soul,
  title={{SOUL}: Unlocking the Power of Second-Order Optimization for {LLM} Unlearning},
  author={Jia, Jinghan and others},
  booktitle={Proceedings of the 2024 Conference on Empirical Methods in Natural Language Processing},
  pages={4276--4292},
  year={2024}
}

@inproceedings{gao2025continual,
  title={On Large Language Model Continual Unlearning},
  author={Gao, Chongyang and Wang, Lixu and Ding, Kaize and Weng, Chenkai and Wang, Xiao and Zhu, Qi},
  booktitle={International Conference on Learning Representations},
  year={2025}
}

@inproceedings{pawelczyk2024incontext,
  title={In-Context Unlearning: Language Models as Few-Shot Unlearners},
  author={Pawelczyk, Martin and Neel, Seth and Lakkaraju, Himabindu},
  booktitle={Proceedings of the 41st International Conference on Machine Learning},
  series={Proceedings of Machine Learning Research},
  volume={235},
  pages={40034--40050},
  year={2024}
}

@inproceedings{bhaila2025spul,
  title={Soft Prompting for Unlearning in Large Language Models},
  author={Bhaila, Karuna and Van, Minh-Hao and Wu, Xintao},
  booktitle={Proceedings of the 2025 Conference of the Nations of the Americas Chapter of the Association for Computational Linguistics: Human Language Technologies (Volume 1: Long Papers)},
  pages={4046--4056},
  year={2025}
}

@inproceedings{muresanu2025fast,
  title={Fast Exact Unlearning for In-Context Learning Data for {LLM}s},
  author={Muresanu, Andrei Ioan and Thudi, Anvith and Zhang, Michael R. and Papernot, Nicolas},
  booktitle={Proceedings of the 42nd International Conference on Machine Learning},
  series={Proceedings of Machine Learning Research},
  volume={267},
  pages={45272--45288},
  year={2025}
}

@inproceedings{saha2021gradient,
  title={Gradient Projection Memory for Continual Learning},
  author={Saha, Gobinda and Garg, Isha and Roy, Kaushik},
  booktitle={International Conference on Learning Representations},
  pages={944--961},
  year={2021}
}

@inproceedings{li2024wmdp,
  title={The {WMDP} Benchmark: Measuring and Reducing Malicious Use With Unlearning},
  author={Li, Nathaniel and others},
  booktitle={Proceedings of the 41st International Conference on Machine Learning},
  series={Proceedings of Machine Learning Research},
  volume={235},
  pages={28525--28550},
  year={2024}
}

@inproceedings{hendrycks2021mmlu,
  title={Measuring Massive Multitask Language Understanding},
  author={Hendrycks, Dan and others},
  booktitle={International Conference on Learning Representations},
  year={2021}
}

@article{touvron2023llama,
  title={Llama 2: Open Foundation and Fine-Tuned Chat Models},
  author={Touvron, Hugo and others},
  journal={arXiv preprint arXiv:2307.09288},
  year={2023}
}

@article{tunstall2023zephyr,
  title={Zephyr: Direct Distillation of {LM} Alignment},
  author={Tunstall, Lewis and others},
  journal={arXiv preprint arXiv:2310.16944},
  year={2023}
}

@article{zhang2024tinyllama,
  title={TinyLlama: An Open-Source Small Language Model},
  author={Zhang, Peiyuan and Zeng, Guangtao and Wang, Tianduo and Lu, Wei},
  journal={arXiv preprint arXiv:2401.02385},
  year={2024}
}

@article{wang2025erasing,
  title={Erasing Without Remembering: Implicit Knowledge Forgetting in Large Language Models},
  author={Wang, Huazheng and others},
  journal={arXiv preprint arXiv:2502.19982},
  year={2025}
}

@inproceedings{jia2024wagle,
  title={{WAGLE}: Strategic Weight Attribution for Effective and Modular Unlearning in Large Language Models},
  author={Jia, Jinghan and Liu, Jiancheng and Zhang, Yihua and Ram, Parikshit and Baracaldo, Nathalie and Liu, Sijia},
  booktitle={Advances in Neural Information Processing Systems},
  year={2024}
}

@inproceedings{fan2025simplicity,
  title={Simplicity Prevails: Rethinking Negative Preference Optimization for {LLM} Unlearning},
  author={Fan, Chongyu and others},
  booktitle={Advances in Neural Information Processing Systems},
  year={2025}
}

\end{document}